\pdfoutput=1
\documentclass[letterpaper]{article}
\usepackage[totalwidth=480pt, totalheight=680pt]{geometry}

\usepackage{authblk}

\usepackage{algorithm}
\usepackage{algorithmic}
\usepackage{natbib}
\usepackage{enumerate}

\usepackage{microtype}
\usepackage{graphicx}
\usepackage{subfigure}
\usepackage{booktabs}

\usepackage{amsmath}
\usepackage{amssymb}
\usepackage{mathtools}
\usepackage{amsthm}

\usepackage{xcolor}
\definecolor{darkgreen}{rgb}{0,0.5,0}
\usepackage{hyperref}
\hypersetup{
    unicode=false,          
    colorlinks=true,        
    linkcolor=red,          
    citecolor=darkgreen,    
    filecolor=magenta,      
    urlcolor=blue           
}
\usepackage[capitalize, nameinlink]{cleveref}

\usepackage{bm}
\usepackage{cleveref}
\usepackage{thm-restate}
\theoremstyle{plain}
\newtheorem{theorem}{Theorem}[section]

\newtheorem{lemma}[theorem]{Lemma}

\theoremstyle{definition}

\theoremstyle{remark}

\usepackage{pifont}
\theoremstyle{plain}
\newtheorem{goal}[theorem]{Goal}

\usepackage{tikz}
\usetikzlibrary{calc, graphs, graphs.standard, shapes, arrows, arrows.meta, positioning, decorations.pathreplacing, decorations.markings, decorations.pathmorphing, fit, matrix, patterns, shapes.misc, tikzmark}
\usetikzlibrary{math}

\newcommand{\cA}{\mathcal{A}}

\newcommand{\cG}{\mathcal{G}}

\newcommand{\cO}{\mathcal{O}}

\newcommand{\eps}{\varepsilon}
\newcommand{\E}{\mathbb{E}}

\renewcommand{\algorithmiccomment}[1]{\hfill$\rhd$ #1}

\title{Online bipartite matching with imperfect advice}
\author[1]{Davin Choo\thanks{Equal contribution}}
\author[2]{Themis Gouleakis$^*$}
\author[3]{Chun Kai Ling$^*$}
\author[4]{Arnab Bhattacharyya}
\affil[1,2,4]{School of Computing, National University of Singapore}
\affil[3]{Industrial Engineering and Operations Research, Columbia University}
\date{}

\begin{document}

\maketitle

\begin{abstract}
We study the problem of online unweighted bipartite matching with $n$ offline vertices and $n$ online vertices where one wishes to be competitive against the optimal offline algorithm.
While the classic \textsc{Ranking} algorithm of \cite{karp1990optimal} provably attains competitive ratio of $1-1/e > 1/2$, we show that no learning-augmented method can be both 1-consistent and strictly better than $1/2$-robust under the adversarial arrival model.
Meanwhile, under the random arrival model, we show how one can utilize methods from distribution testing to design an algorithm that takes in external advice about the online vertices and provably achieves competitive ratio interpolating between any ratio attainable by advice-free methods and the optimal ratio of 1, depending on the advice quality.
\end{abstract}

\section{Introduction}
\label{sec:intro}

Finding matchings in bipartite graphs is a mainstay of algorithms research. The area's mathematical richness is complemented by a vast array of applications --- any two-sided market (e.g., kidney exchange, ridesharing) yields a matching problem. 
In particular, the \emph{online} variant enjoys much attention due to its application in internet advertising.
Consider a website with a number of pages and ad slots (videos, images, etc.). Advertisers specify ahead of time the pages and slots they like their ads to appear in, as well as the target user. The website is paid based on the number of ads appropriately fulfilled.
Crucially, ads slots are available only when traffic occurs on the website and are not known in advance. Thus, the website is faced with the \emph{online} decision problem of matching advertisements to open ad slots.

The classic online unweighted bipartite matching problem by \citet{karp1990optimal} features $n$ offline vertices $U$ and $n$ online vertices $V$.
Each $v \in V$ reveals its incident edges sequentially upon arrival.
With each arrival, one makes an irrevocable decision whether (and how) to match $v$ with a neighboring vertex in $U$.
The final offline graph $\cG = (U \cup V, E)$ is assumed to have a largest possible matching of size $n^* \leq n$, and we seek online algorithms producing matchings of size as close to $n^*$ as possible.
The performance of a (randomized) algorithm $\cA$ is measured by its \emph{competitive ratio}:
\begin{equation}
\label{eq:competitive-ratio-adversarial}
\min_{\cG= (U \cup V, E)} \min_{\substack{V\text{'s arrival seq.}}} \frac{\E[\text{\# matches by $\cA$}]}{n^*} \;,
\end{equation}
where the randomness is over any random decisions made by $\cA$.
Traditionally, one assumes the \emph{adversarial arrival model}, i.e., an adversary controls both the final graph $\cG$ and the arrival sequence of online vertices.

Since any maximal matching has size at least $n^*/2$, a greedy algorithm trivially attains a competitive ratio of $1/2$. Indeed, \citet{karp1990optimal} show that no deterministic algorithm can guarantee better than $1/2 - o(1)$.
Meanwhile, the randomized \textsc{ranking} algorithm of \citet{karp1990optimal} attains an asymptotic competitive ratio of $1-1/e$ which is also known to be optimal \citep{karp1990optimal,goel2008online,birnbaum2008line,vazirani2022online}.

In practice, \emph{advice} (also called predictions or side information) is often available for these online instances.
For example, online advertisers often aggregate past traffic data to estimate the \emph{future} traffic and corresponding user demographic.
While such advice may be imperfect, it may nonetheless be useful in increasing revenue and improving upon aforementioned worst-case guarantees.
Designing algorithms that utilize such advice in a principled manner falls under the research paradigm of \emph{learning-augmented algorithms}.
A learning-augmented algorithm is said to be (i) $a$-\emph{consistent} if it is $a$-\emph{competitive} with perfect advice and (ii) $b$-\emph{robust} if it is $b$-\emph{competitive} with arbitrary advice quality.

\begin{goal}
\label{goal:test}
Let $\beta$ be the best-known competitive ratio attainable by any classical advice-free online algorithm.
Can we design a learning-augmented algorithm for the online bipartite matching problem that is 1-consistent and $\beta$-robust?
\end{goal}

Clearly, \cref{goal:test} depends on the form of advice as well as a suitable measure of its quality. Setting these technicalities aside for now, we remark that \cref{goal:test} strikes the best of all worlds: it requires that a perfect matching be obtained when the advice is perfect, while not sacrificing performance with respect to advice-free algorithms when faced with low-quality advice. In other words, there is potential to benefit, but no possible harm when employing such an algorithm. We make the following contributions in pursuit of \cref{goal:test}.

\paragraph{1. Impossibility under adversarial arrivals}\hspace{0pt}\\
We show that under adversarial arrivals, learning augmented algorithms, no matter what form the advice takes, cannot be both 1-consistent and strictly more than 1/2-robust. The latter is worse than the competitive ratio of $1-1/e$ guaranteed by known advice-free algorithms \citep{karp1990optimal}.

\paragraph{2. Achieving \cref{goal:test} under the random arrival model}\hspace{0pt}\\
We propose an algorithm \textsc{TestAndMatch} achieving \cref{goal:test} under the weaker \emph{random arrival model}, in which an adversary controls the online vertices $V$ but its arrival order is randomized.
Our advice is a \emph{histogram over types} of online vertices; in the context of online advertising this corresponds to a forecast of the user demographic and which ads they can be matched to.
\textsc{TestAndMatch} assumes perfect advice while simultaneously testing for its accuracy via the initial arrivals.
If the advice is deemed useful, we mimic the matching suggested by it; else, we revert to an advice-free method.
The testing phase is kept short (sublinear in $n$) by utilizing state-of-the-art $L_1$ estimators from distribution testing.
We analyze our algorithm's performance as a function of the quality of advice, showing that its competitive ratio gracefully degrades to $\beta$ as quality of advice decays.
To the best of our knowledge, our work is the first that shows how one can leverage techniques from the property testing literature to designing learning-augmented algorithms.

\hspace{0pt}\\
While our contributions are mostly theoretical, we give and discuss various practical extensions of \textsc{TestAndMatch}, and also show preliminary experiments in \cref{sec:appendix-poc}.

\section{Preliminaries and related work}
\label{sec:prelims}

An online bipartite instance is defined by a bipartite graph $\cG = (U \cup V, E)$ where $U$ and $V$ are the set of $n$ offline and $n$ online vertices respectively.
A \emph{type} of an online vertex $v \in V$ refers to the subset of offline vertices $\{ u \in U \mid \{u,v\} \in E \}$ that it is neighbors with; there are $2^n$ possible types and at most $n$ of them are realized through $V$.
The types of vertices $v_i \in V$ are revealed one at a time in an online fashion, when the corresponding vertex arrives and one has to decide whether (and how) to match the newly arrived vertex irrevocably.
A \emph{matching} in the graph $\cG$ is a set of edges $M \subseteq E$ such that for every vertex $w \in U \cup V$, there is at most one edge in $M$ incident to $w$. Given two vectors of length $k$, we denote the $L_1$-distance between them as $L_1(x,y)=||x-y||_1=\sum_{i=1}^k |x_i-y_i|$. For any set $S$, $2^S$ denotes its power set (set of all subsets of $S$).

In this work, we focus on the classic unweighted online bipartite matching (see \citet{mehta2013online} for other variants) where the final offline graph has a matching of size $n^* \leq n$.

\paragraph{Arrival models.}
The degree of control an adversary has over $V$ affects analysis and algorithms.
The \emph{adversarial arrival model} is the most challenging, with both the final graph $\cG$ and the order in which online vertices arrive chosen by the adversary. Here, an algorithm's competitive ratio is given by \eqref{eq:competitive-ratio-adversarial}.
In \emph{random arrival models}, $\cG$ remains adversarial but the arrival order is random.
For this paper, we assume the \emph{Random Order} setting, where an adversary chooses a $\cG$, but the arrival order of $V$ is a uniformly random permutation.
In this setting, the competitive ratio is defined as
\begin{align}
\label{eq:competitive-ratio-random-arrival}
\min_{\cG= (U \cup V, E)} \E_{\substack{V\text{'s arrival seq.}}} \frac{\E[\text{\# matches by $\cA$}]}{n^*} \;.
\end{align}
Two even easier random arrival models exist: (i) \emph{known-IID model} \citep{feldman2009online}, where the adversary chooses a distribution over types (which is known to us), and the arrivals of $V$ are chosen by sampling i.i.d. from this distribution, and (ii) \emph{unknown-IID model}, which is the same as known-IID but with the distributions are not revealed to us.
The competitive ratios between these arrival models are known to exhibit a hierarchy of difficulty \citep{mehta2013online}:
\begin{align*}
    \text{Adversarial}
    \leq \text{Rand. Order}
    \leq \text{Unknown-IID}
    \leq \text{Known-IID}
\end{align*}
As our Random Order setting is the most challenging amongst these random arrival models, our methods also apply to the unknown-IID and known-IID settings.

\subsection{Advice-free online bipartite matching}

\cref{tab:competitive-ratios} summarizes known results about attainable competitive ratios and impossibility results in the adversarial and Random Order arrival models.
In particular, observe that there is a gap between the upper and lower bounds in the Random Order arrival model which remains unresolved.

\begin{table}[h]
\centering
\begin{tabular}{ccc}
\toprule
& Adversarial & Random Order\\
\midrule
det.\ algo.\ & $1/2$ & $1 - 1/e$\\
det.\ hardness & $1/2$ & $3/4$\\
rand.\ algo.\ & $1 - 1/e$ & $0.696$\\
rand.\ hardness & $1 - 1/e + o(1)$ & $0.823$\\
\bottomrule
\end{tabular}
\caption{Known competitive ratios for the classic unweighted online bipartite matching problem for deterministic (det.) and randomized (rand.) algorithms under the adversarial and Random Order arrival models.
Note that $1 - 1/e \approx 0.63$.}
\label{tab:competitive-ratios}
\end{table}

On the positive side of things, the deterministic \textsc{Greedy} algorithm which matches newly arrived vertex with any unmatched offline neighbor attains a competitive ratio of at least $1/2$ in the adversarial arrival model and at least $1 - 1/e$ in the random arrival model \citep{goel2008online}.
Meanwhile, the randomized \textsc{Ranking} algorithm of \citet{karp1990optimal} achieves a competitive ratio of $1 - 1/e$ in the adversarial arrival model.
In the Random Order arrival model, \textsc{Ranking} achieves a strictly larger competitive ratio, shown to be at least 0.653 in \citet{karande2011online} and 0.696 in \citet{mahdian2011online}.
However, \citet{karande2011online} showed that \textsc{Ranking} cannot beat 0.727 in general; so, new ideas will be required if one believes that the tight competitive ratio bound is 0.823 \citep{manshadi2012online}.

On the negative side, the following example highlights the key difficulty faced by online algorithms.
Consider the gadget for $n=2$ in \cref{fig:gadget}, where the first online vertex $v_1$ neighbors with both $u_1$ and $u_2$ and the second online vertex $v_2$ neighbors with only one of $u_1$ or $u_2$.
Even when promised that the true graph is either $\cG_1$ or $\cG_2$, any online algorithm needs to correctly guess whether to match $v_1$ with $u_1$ or $u_2$ to achieve perfect matching when $v_2$ arrives.

\begin{figure}[h]
\centering
\resizebox{0.5\linewidth}{!}{
    \begin{tikzpicture}[square/.style={regular polygon,regular polygon sides=4}]

\node[] at (-1.5,-0.5) {$\cG_1$};

\node[draw, thick, square, minimum size=10pt, inner sep=1pt] at (0,0) (u1) {};
\node[left=5pt of u1] {$u_1$};
\node[draw, thick, square, minimum size=10pt, inner sep=1pt] at (0,-1) (u2) {};
\node[left=5pt of u2] {$u_2$};

\node[draw, thick, circle, minimum size=10pt, inner sep=1pt] at (1,0) (v1) {};
\node[right=5pt of v1] {$v_1$};
\node[draw, thick, circle, minimum size=10pt, inner sep=1pt] at (1,-1) (v2) {};
\node[right=5pt of v2] {$v_2$};

\draw[thick] (v1) -- (u1);
\draw[thick] (v1) -- (u2);
\draw[thick, red] (v2) -- (u1);

\draw[thick, dashed] (-0.5,-0.5) -- (4.5,-0.5);

\node[] at (5.5,-0.5) {$\cG_2$};

\node[draw, thick, square, minimum size=10pt, inner sep=1pt] at (3,0) (cu1) {};
\node[left=5pt of cu1] {$u_1$};
\node[draw, thick, square, minimum size=10pt, inner sep=1pt] at (3,-1) (cu2) {};
\node[left=5pt of cu2] {$u_2$};

\node[draw, thick, circle, minimum size=10pt, inner sep=1pt] at (4,0) (cv1) {};
\node[right=5pt of cv1] {$v_1$};
\node[draw, thick, circle, minimum size=10pt, inner sep=1pt] at (4,-1) (cv2) {};
\node[right=5pt of cv2] {$v_2$};

\draw[thick] (cv1) -- (cu1);
\draw[thick] (cv1) -- (cu2);
\draw[thick, red] (cv2) -- (cu2);
\end{tikzpicture}
}
\caption{Gadget for $n=2$. Red edges observed when $v_2$ arrives.}
\label{fig:gadget}
\end{figure}
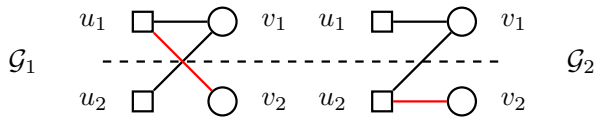

By repeating the gadget of \cref{fig:gadget} multiple times sequentially, any deterministic algorithm can only hope to attain competitive ratios of $1/2$ and $3/4$ in the adversarial and random arrival models respectively.
For randomized algorithms, \cite{karp1990optimal} showed that \textsc{Ranking} is essentially optimal for the adversarial arrival model since no algorithm can achieve a competitive ratio better than $1 - 1/e + o(1)$.
In the Random Order arrival model, \citet{goel2008online} showed (in their Appendix E) 
that a ratio better than $5/6 \approx 0.83$ cannot be attained by brute force analysis of a $3 \times 3$ gadget bipartite graph.
Subsequently, \citet{manshadi2012online} showed that no algorithm (deterministic or randomized) can achieve a competitive ratio better than $0.823$.

Technically speaking, the hardness result of \citet{manshadi2012online} is for the known IID \ model introduced by \citet{feldman2009online}, but this extends to the Random Order arrival model since the former is an easier setting; e.g.\ see Theorem 2.1 in \citep{mehta2013online} for an explanation.
Under the easier known IID model, the current state of the art algorithms achieve a competitive ratio of 0.7299 using linear programming \citep{jaillet2014online,brubach2016new,brubach2020online}.

\subsection{Learning-augmented algorithms for matching}
\label{sec:learning-augmented-related-work}


Learning-augmented algorithms as a whole have received significant attention since the seminal work of \citet{lykouris2021competitive}, where they investigated the online caching problem with predictions; their result was further improved by \citet{rohatgi2020near,antoniadis2020online,wei2020better}.
Algorithms with advice was also studied for the ski-rental problem \citep{gollapudi2019online,wang2020online,angelopoulos2020online}, non-clairvoyant scheduling \citep{purohit2018improving}, scheduling \citep{lattanzi2020online,bamas2020learning,antoniadis2022novel}, augmenting classical data structures with predictions (e.g.\ indexing \citep{kraska2018case} and Bloom filters \citep{mitzenmacher2018model}), online selection and matching problems \citep{antoniadis2020secretary, dutting2021secretaries}, online TSP \citep{bernardiniuniversal,gouleakis2023learning}, a more general framework of online primal-dual algorithms \citep{bamas2020primal}, and causal graph learning \citep{choo2023active}.

\citet{aamand2022optimal} studied the adversarial arrival models with offline vertex degrees as advice.
While their algorithm is optimal under the Chung-Lu-Vu random graph model \citep{chung2003spectra}, the class of offline degree advice is unable to attain 1-consistency.
\citet{feng2021two} propose a two-stage vertex-weighted variant, where advice is a proposed matching for the online vertices arriving in the first stage. \citet{jin2022online} showed in this setting a tight robustness-consistency tradeoff and derive a continuum of algorithms tracking this Pareto frontier. 
\citet{antoniadis2020secretary} studied settings with random vertex arrival and weighted edges.
Their advice is a prediction on edge weights adjacent to $V$ under an optimal offline matching.
Furthermore, their algorithm and analysis uses a hyper-pamareter quatifying confidence in the advice, leading to different consistency and robustness tradeoffs.
Another relevant work is the \textsc{LOMAR} method proposed by \citet{li2023learning}.
Using a pre-trained reinforcement learning (RL) model along with a switching mechanism based on regret to guarantee robustness with respect to any provided expert algorithm, they claim ``for some tuning parameter $\rho \in [0,1]$, \textsc{LOMAR} is $\rho$-competitive against our choice of expert online algorithm''.
We differ from \textsc{LOMAR} in two key ways:
\begin{enumerate}
    \item Our method does not require any pre-training phase and directly operate on the sequence of online vertices themselves.
    This means that whatever mistakes made during our ``testing'' phase contributes to our competitive ratio; a key technical contribution is the use of distribution testing to ensure that the number of such mistakes incurred is sublinear.
    \item The robustness guarantee of \citet{li2023learning} is substantially weaker than what we provide.
    Suppose the expert used by \textsc{LOMAR} is $\beta$-competitive, just like how we use the state-of-the-art algorithm as the baseline.
    Although \citet{li2023learning} does not analyze the consistency guarantee of their method, one can see that LOMAR is $(1 - \rho)$-consistent and $\rho \cdot \beta$-robust (ignoring the $B \geq 0$ hyperparameter).
    \textsc{LOMAR} can only be 1-consistent when $\rho = 0$, i.e.\ it blindly follows the RL-based method; but then it will have no robustness guarantees.
    In other words, \textsc{LOMAR} cannot simultaneously achieve 1-consistency and $\rho \cdot \beta$-robustness without knowing the RL quality.
    In contrast, our method is \emph{simultaneously} 1-consistent and $\approx \beta$-robust \emph{without} knowing the quality of our given advice; we evaluate its quality as vertices arrive.
\end{enumerate}
\cref{tab:literature-comparision} compares the consistency-robustness tradeoffs.

\begin{table}[h]
\centering
\begin{tabular}{cccc}
\toprule
& \citep{jin2022online} & \textsc{LOMAR} & Ours\\
\midrule
Robustness & $R$ & $\rho \cdot \beta$ & $\approx \beta$\\
Consistency & $1 - (1 - \sqrt{1-R})^2$ & $1 - \rho$ & 1 \\
\bottomrule
\end{tabular}
\caption{Consistency-robustness guarantees of methods that can achieve 1-consistency.
Here, $R \in [0, 3/4]$ and $\rho \in [0,1]$.
Note that \citep{jin2022online} is for the 2-staged setting.}
\label{tab:literature-comparision}
\end{table}

More broadly, \citet{lavastida2020learnable,lavastida2021using} learn and exploit parameters of the online matching problem and provide PAC-style guarantees.
\citet{dinitz2022algorithms} studied the use of multiple advice and seek to compete with the best on a per-instance basis. Finally, others suggest using advice to speedup offline matching via ``warm-start'' heuristics \citep{dinitz2021faster,chen2022faster,sakaue2022discrete}. 

\subsection{Distribution testing and distance estimation}\label{sec:distribution-testing-related-work}

In this work, we will use results from \citet{JHW18} for the problem of $L_1$ distance estimation.
This is closely related to \emph{tolerant identity testing}, where the tester's task is to distinguish whether a distribution $p$ is $\eps_1$-close to some known distribution $q$ from the case where $p$ is $\eps_2$-far from $q$, according to some natural distance measure.

The following theorem states the number of samples from an unknown distribution $p$ that needed by the algorithm in \cite{JHW18} to get an estimate of $L_1(p,q)$ for some reference distribution $q$ with additive error $\eps$ and error probability $\delta$.\footnote{It is our understanding that the tester proposed by \citet{JHW18} requires a significant amount of hyperparameter tuning and no off-the-shelf implementation is available \cite{JHW-email}.}

\begin{restatable}[adapted from \cite{JHW18}]{theorem}{minimax}
\label{thm:minimax}
Fix a reference distribution $q$ over a domain $T$ of size $|T| = r$ and let $s \in \cO \left(\frac{r \cdot \log(1/\delta)}{\eps^2 \cdot \log r} \right)$ be an even integer.
There exists an algorithm that draws $s_1 + s_2$ IID samples from an unknown distribution $p$ over $T$, where $s_1, s_2 \sim \mathrm{Poisson}(s/2)$, and outputs an estimate $\hat{L}_1$ such that $|\hat{L}_1 - L_1(p,q)| \leq \eps$ with success probability at least $1-\delta$.
\end{restatable}

The algorithm of \cref{thm:minimax} uses a standard technique in distribution testing known as \emph{Poissonization} which aims to eliminate correlations between samples at the expense of not having a fixed sample size.
Instead, the number of samples follows a Poisson distribution and we treat its mean as the sample complexity.
As a consequence, the known result regarding the concentration of the Poisson distribution would be helpful in bounding the overall algorithmic success probability, e.g.\ see \citep{canonne2019short}.

\begin{lemma}
\label{lem:poisson-bound}
For any $m > 0$ and any $x > 0$, we have $\Pr[\vert X-m\vert\geq x]\leq 2e^{-\frac{x^2}{2(m+x)}}$, where $X \sim \mathrm{Poisson}(m)$.
\end{lemma}

\section{Impossibility for adversarial arrival model}
\label{sec:hardness}

Unfortunately, \cref{goal:test} is unattainable under the adversarial arrival model.
Our construction is based on generalizing the gadget in \cref{fig:gadget} to state \cref{thm:hardness}.

\begin{restatable}{theorem}{hardness}
\label{thm:hardness}
For even $n$, there exists input graphs $\cG_1$ and $\cG_2$ such that no advice can distinguish between the two within $n/2$ online arrivals. Consequently, an algorithm \emph{cannot} be both 1-consistent and strictly more than 1/2-robust.
\end{restatable}
\begin{proof}
Consider the restricted case where there are only two possible final offline graphs $\cG^{(1)} =(U \cup V^{(1)}, E^{(1)})$ and $\cG^{(2)} =(U \cup V^{(2)}, E^{(2)})$ where
\[
E^{(1)}
= \left\{ \{u^{(1)}_{j}, v^{(1)}_{j}\}, \{u^{(1)}_{j+n/2}, v^{(1)}_{j}\} : 1 \leq j \leq n/2 \right\}
\cup \left\{ {\color{blue}\{u^{(1)}_{j-n/2}, v^{(1)}_{j}\}} : n/2+1 \leq j \leq n \right\}
\]
\[
E^{(2)}
= \left\{ \{u^{(2)}_{j}, v^{(2)}_{j}\}, \{u^{(2)}_{j+n/2}, v^{(2)}_{j}\} : 1 \leq j \leq n/2 \right\}\\
\cup \left\{ {\color{blue}\{u^{(2)}_{j}, v^{(2)}_{j}\}} : n/2+1 \leq j \leq n \right\}    
\]
We will even restrict the first $n/2$ to be exactly $v^{(i)}_1, \ldots, v^{(i)}_{n/2}$, where $i \in \{1,2\}$ is the chosen input graph by the adversary.
See \cref{fig:hardness} for an illustration.

\begin{figure}[htb]
\centering
\resizebox{0.7\linewidth}{!}{
    \begin{tikzpicture}[square/.style={regular polygon,regular polygon sides=4}]
\tikzmath{\graphgap=0.75;}
\node[] at (1,0.75*\graphgap) {$\cG_1$};

\node[draw, thick, square, minimum size=10pt, inner sep=1pt] at (0,0) (u1) {};
\node[left=1pt of u1] {$u_1$};
\node[] at (0,-1*\graphgap) (udots) {$\vdots$};
\node[draw, thick, square, minimum size=10pt, inner sep=1pt] at (0,-2*\graphgap) (uhalf) {};
\node[left=1pt of uhalf] {$u_{\frac{n}{2}}$};
\node[draw, thick, square, minimum size=10pt, inner sep=1pt] at (0,-3*\graphgap) (uhalfplusone) {};
\node[] at (0,-4*\graphgap) (udots2) {$\vdots$};
\node[left=1pt of uhalfplusone] {$u_{\frac{n}{2}+1}$};
\node[draw, thick, square, minimum size=10pt, inner sep=1pt] at (0,-5*\graphgap) (un) {};
\node[left=1pt of un] {$u_n$};

\node[draw, thick, circle, minimum size=10pt, inner sep=1pt] at (2,0) (v1) {};
\node[right=1pt of v1] {$v_1$};
\node[] at (2,-1*\graphgap) (vdots) {$\vdots$};
\node[draw, thick, circle, minimum size=10pt, inner sep=1pt] at (2,-2*\graphgap) (vhalf) {};
\node[right=1pt of vhalf] {$u_{\frac{n}{2}}$};
\node[draw, thick, circle, minimum size=10pt, inner sep=1pt] at (2,-3*\graphgap) (vhalfplusone) {};
\node[] at (2,-4*\graphgap) (vdots2) {$\vdots$};
\node[right=1pt of vhalfplusone] {$v_{\frac{n}{2}+1}$};
\node[draw, thick, circle, minimum size=10pt, inner sep=1pt] at (2,-5*\graphgap) (vn) {};
\node[right=1pt of vn] {$v_n$};

\draw[thick] (v1) -- (u1);
\draw[thick] (v1) -- (uhalfplusone);
\draw[thick] (vdots) -- (udots);
\draw[thick] (vdots) -- (udots2);
\draw[thick] (vhalf) -- (uhalf);
\draw[thick] (vhalf) -- (un);
\draw[thick, red] (vhalfplusone) -- (u1);
\draw[thick, red] (vdots2) -- (udots);
\draw[thick, red] (vn) -- (uhalf);

\draw[thick, dashed] (-0.5,-2.5*\graphgap) -- (7.5,-2.5*\graphgap);

\node[] at (6,0.75*\graphgap) {$\cG_2$};

\node[draw, thick, square, minimum size=10pt, inner sep=1pt] at (5,0) (cu1) {};
\node[left=1pt of cu1] {$u_1$};
\node[] at (5,-1*\graphgap) (cudots) {$\vdots$};
\node[draw, thick, square, minimum size=10pt, inner sep=1pt] at (5,-2*\graphgap) (cuhalf) {};
\node[left=1pt of cuhalf] {$u_{\frac{n}{2}}$};
\node[draw, thick, square, minimum size=10pt, inner sep=1pt] at (5,-3*\graphgap) (cuhalfplusone) {};
\node[] at (5,-4*\graphgap) (cudots2) {$\vdots$};
\node[left=1pt of cuhalfplusone] {$u_{\frac{n}{2}+1}$};
\node[draw, thick, square, minimum size=10pt, inner sep=1pt] at (5,-5*\graphgap) (cun) {};
\node[left=1pt of cun] {$u_n$};

\node[draw, thick, circle, minimum size=10pt, inner sep=1pt] at (7,0) (cv1) {};
\node[right=1pt of cv1] {$v_1$};
\node[] at (7,-1*\graphgap) (cvdots) {$\vdots$};
\node[draw, thick, circle, minimum size=10pt, inner sep=1pt] at (7,-2*\graphgap) (cvhalf) {};
\node[right=1pt of cvhalf] {$u_{\frac{n}{2}}$};
\node[draw, thick, circle, minimum size=10pt, inner sep=1pt] at (7,-3*\graphgap) (cvhalfplusone) {};
\node[] at (7,-4*\graphgap) (cvdots2) {$\vdots$};
\node[right=1pt of cvhalfplusone] {$v_{\frac{n}{2}+1}$};
\node[draw, thick, circle, minimum size=10pt, inner sep=1pt] at (7,-5*\graphgap) (cvn) {};
\node[right=1pt of cvn] {$v_n$};

\draw[thick] (cv1) -- (cu1);
\draw[thick] (cv1) -- (cuhalfplusone);
\draw[thick] (cvdots) -- (cudots);
\draw[thick] (cvdots) -- (cudots2);
\draw[thick] (cvhalf) -- (cuhalf);
\draw[thick] (cvhalf) -- (cun);
\draw[thick, red] (cvhalfplusone) -- (cuhalfplusone);
\draw[thick, red] (cvdots2) -- (cudots2);
\draw[thick, red] (cvn) -- (cun);
\end{tikzpicture}
}
\caption{Illustration of $\cG_1$ and $\cG_2$ for \cref{thm:hardness}}
\label{fig:hardness}
\end{figure}
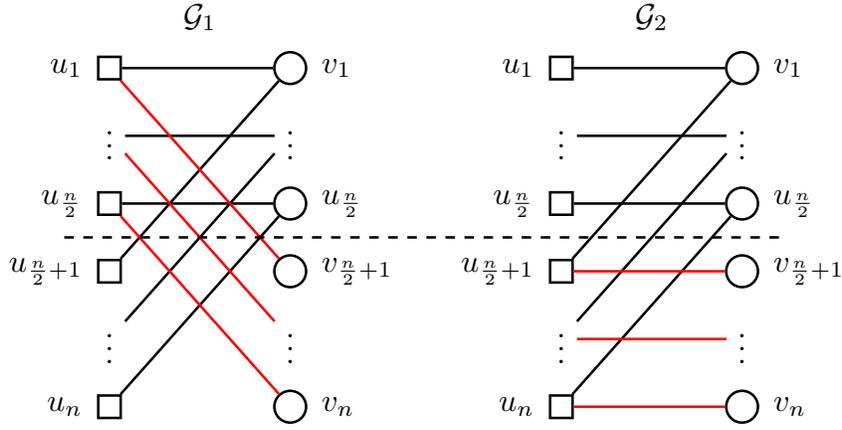

Suppose $\cG_i$ was the chosen graph, for $i \in \{1,2\}$.
In this restricted problem input setting, the strongest possible advice is knowing the bit $i$ since all other viable advice can be derived from this bit.
Thus, for the sake of a hardness result, it suffices to only consider the advice of $\hat{i} \in \{1,2\}$.

Within the first $n/2$ arrivals, any algorithm cannot distinguish and will behave in the same manner.
Suppose there is a 1-consistent algorithm $A$ given bit $\hat{i}$.
In the first $n/2$ steps, $A$ needs to match $v_j$ to $u_{j+n/2}$ if $\hat{i} = 1$ and $v_j$ to $u_j$ for $\hat{i} = 2$.
However, if $i \neq \hat{i}$, then $A$ will not be able to match any remaining arrivals and hence be at most $1/2$-robust.
\end{proof}

In fact, \cref{thm:hardness} can be strengthened: for any $\alpha \in [0, 1/2]$, no algorithm can be simultaneously $(1-\alpha)$-consistent and strictly more than $(1/2+\alpha)$-robust. The proof is essentially identical and deferred to \cref{sec:appendix-extension-hardness}.

While \cref{thm:hardness} appears simple, we stress that hardness results for learning-augmented algorithms are rare, since the form of advice and its utilization is arbitary.
For instance, \citet{aamand2022optimal} only showed that when advice is the true degrees of the offline vertices, there exist inputs such that any learning-augmented algorithm can only achieve a competitive ratio of at most $1 - 1/e + o(1)$.

\section{Imperfect advice for random arrival model}
\label{sec:imperfect-advice}

In this section, we present our learning-augmented algorithm \textsc{TestAndMatch} which is $1$-consistent, $(\beta - o(1))$-robust, and achieves a smooth interpolation on an appropriate notion of advice quality, where $\beta$ is any achieveable competitive ratio by some advice-free baseline algorithm.
As discussed in \cref{sec:prelims}, the best known competitive ratio of $\beta = 0.696$ is achieveable using \textsc{Ranking} \citep{karp1990optimal} but it is unknown if it can be improved.
In fact, \textsc{TestAndMatch} is a meta-algorithm that uses any advice-free baseline algorithm as a blackbox and so our robustness guarantee improves as $\beta$ improves.

\paragraph{Using realized type counts as advice.} 
Given the final offline graph $\cG^* = (U \cup V, E)$ with maximum matching size $n^* \leq n$, we can classify each online vertex based on their \emph{types}, i.e., the set of offline vertices they are adjacent to \citep{borodin2020experimental}.
Define the vector $c^* \in \mathbb{Z}^{2^n}$ indexed by the possible types $2^U$, such that $c^*(t)$ is the number of times type $t \in 2^U$ occurs in $\cG^*$.
Even though there are $2^n$ possible types, the number of \emph{realized} types is at most $n$. 
Let $T^* \subseteq 2^U$ be the set of types with non-zero counts in $c^*$.
Since $|U|=|V|=n$, $c^*$ is sparse and contains $r^* = |T^*| \leq n \ll 2^n$ non-zero elements; see \cref{fig:type-example}.
Note that $c^*$ fully determines $\cG^*$ for our purposes, as vertices may be permuted but $n^*$ remains identical.

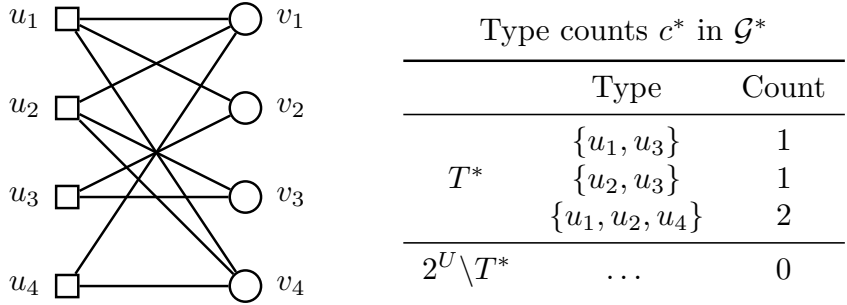
\begin{figure}[h]
\centering
\resizebox{0.7\linewidth}{!}{
    \begin{tikzpicture}[square/.style={regular polygon,regular polygon sides=4}]
\node[draw, thick, square, minimum size=10pt, inner sep=1pt] at (0,0) (u1) {};
\node[left=1pt of u1] {$u_1$};
\node[draw, thick, square, minimum size=10pt, inner sep=1pt] at (0,-1) (u2) {};
\node[left=1pt of u2] {$u_2$};
\node[draw, thick, square, minimum size=10pt, inner sep=1pt] at (0,-2) (u3) {};
\node[left=1pt of u3] {$u_3$};
\node[draw, thick, square, minimum size=10pt, inner sep=1pt] at (0,-3) (u4) {};
\node[left=1pt of u4] {$u_4$};

\node[draw, thick, circle, minimum size=10pt, inner sep=1pt] at (2,0) (v1) {};
\node[right=1pt of v1] {$v_1$};
\node[draw, thick, circle, minimum size=10pt, inner sep=1pt] at (2,-1) (v2) {};
\node[right=1pt of v2] {$v_2$};
\node[draw, thick, circle, minimum size=10pt, inner sep=1pt] at (2,-2) (v3) {};
\node[right=1pt of v3] {$v_3$};
\node[draw, thick, circle, minimum size=10pt, inner sep=1pt] at (2,-3) (v4) {};
\node[right=1pt of v4] {$v_4$};

\draw[thick] (v1) -- (u1);
\draw[thick] (v1) -- (u2);
\draw[thick] (v1) -- (u4);
\draw[thick] (v2) -- (u1);
\draw[thick] (v2) -- (u3);
\draw[thick] (v3) -- (u2);
\draw[thick] (v3) -- (u3);
\draw[thick] (v4) -- (u1);
\draw[thick] (v4) -- (u2);
\draw[thick] (v4) -- (u4);

\node[] at (6.25,-0.15) {Type counts $c^*$ in $\cG^*$};
\node[] at (6.25,-1.75) {
\begin{tabular}{ccc}
\toprule
& Type & Count\\
\midrule
& $\{u_1, u_3\}$ & 1\\
$T^*$ & $\{u_2, u_3\}$ & 1\\
& $\{u_1, u_2, u_4\}$ & 2\\
\midrule
$2^U \backslash T^*$ & \dots & 0\\
\end{tabular}
};
\end{tikzpicture}
}
\caption{For $n=4$, there may be $2^4 = 16$ possible types but at most $n = 4$ of them can ever be non-zero.
Here, $c^*(\{ u_1, u_3 \}) = 1, c^*(\{ u_2, u_3 \}) = 1$ and $c^*(\{ u_1, u_2, u_4\}) = 2$.
We see that type $\{u_1, u_2, u_4\}$ appears twice in $c^*$ and $|T^*|=3$.}
\label{fig:type-example}
\end{figure}

In this work, we consider advice to be an estimate of the \emph{realized type counts} $\hat{c} \in \mathbb{Z}^{2^n}$
with non-zero entries in $\hat{T} \subseteq 2^U$.
As before, we assume that $\hat{c}$ sums to $n$ and contains $\hat{r} = |\hat{T}| \leq n \ll 2^n$ non-zero entries.
Just like $c^*$, $\hat{c}$ fully defines some ``advice graph'' $\hat{\cG}=(U \cup V, \hat{E})$ that we can find a maximum matching for in polynomial time.
We discuss the practicality of obtaining such advice in \cref{sec:appendix-examples-realized-counts-advice}.

Throughout this section, we will use star $(\cdot)^*$ and hat $\hat{(\cdot)}$ to denote ground truth and advice quantities respectively.
In particular, we use $n^* \leq n$ and $\hat{n} \leq n$ to denote the maximum matching size in the final offline graph $\cG^*$ and advice graph $\hat{\cG}$ respectively.
Note that star $(\cdot)^*$ quantities are not known and exist purely for the purpose of analysis.

\begin{algorithm}[t]
\caption{$\textsc{TestAndMatch}$}
\label{alg:testandmatch}
\begin{algorithmic}[1]
    \INPUT Advice $\hat{c}$ with $\hat{r} = |\hat{T}|$, \textsc{Baseline} advice-free algorithm with competitive ratio $\beta < 1$, error threshold $\eps > 0$, failure rate $\delta = \delta' + \delta_{poi}$ for $\delta_{poi}=O\left(\frac{1}{poly(\hat{r})}\right)$
    \STATE Compute advice matching $\hat{M}$ from $\hat{c}$
    \IF{$\frac{\hat{n}}{n} \leq \beta$}
        \STATE Run \textsc{Baseline} on all arrivals
    \ENDIF
    \STATE Define $s_{\hat{r},\eps,\delta} = \cO \left( \frac{(\hat{r}+1) \cdot \log(1/\delta')}{\eps^2 \cdot \log (\hat{r}+1)} \right)$
    \STATE Define testing threshold $\tau = 2 \left( \frac{\hat{n}}{n} - \beta \right) - \eps$
    \STATE Run \textsc{Mimic} on $s_{\hat{r},\eps,\delta} \cdot \sqrt{\log (\hat{r}+1)}$ arrivals while keeping track of the online arrivals in a set $A$
    \IF{$\textsc{MinimaxTest}(s_{\hat{r},\eps,\delta}, \frac{\hat{c}}{n}, A, \tau, \delta')=$ ``Pass''}
        \STATE Run \textsc{Mimic} on the remaining arrivals
    \ELSE
        \STATE Run \textsc{Baseline} on the remaining arrivals
    \ENDIF
\end{algorithmic}
\end{algorithm}

\paragraph{Intuition behind \textsc{TestAndMatch}.}
If $\hat{c} = c^*$, one trivially obtains a 1-consistency by solving for a maximum matching $\hat{M}$ on the advice graph $\hat{\cG}$ and then mimicking matches based on $\hat{M}$ as vertices arrive.
While $\hat{c} \neq c^*$ in general, we may consider distributions $p^*=c^*/n$ and $q=\hat{c}/n$ and test if $p^*$ is close to $q$ in $L_1$ distance via \cref{thm:minimax}; 
this is done sample efficiently using just the first $o(n)$ online vertices (\cref{sec:estimating-advice-quality}).
If $L_1(p^*, q)$ is less than some threshold $\tau$, we conclude $\hat{c} \approx c^*$ and continue mimicking $\hat{M}$, enjoying a competitive ratio close to $1$. If not, we revert to \textsc{Baseline}. Crucially, each wrong match made during the testing phase hurts our final matching size by at most a constant, yielding a competitive ratio of $\beta - o(1)$.

\textsc{TestAndMatch} is described in \cref{alg:testandmatch}, which takes as input a number of additional parameters ($\delta$, $\epsilon$, etc) and subroutines that we will explain in a bit.
For now, we state our main result describing the performance of \textsc{TestAndMatch} in terms of the competitive ratio.

\begin{theorem}
\label{thm:main-result}
For any advice $\hat{c}$ with $|\hat{T}| = \hat{r}$, $\eps > 0$ and $\delta > \frac{1}{poly(\hat{r})}$, let $\hat{L}_1$ be the estimate of $L_1(p^*,q)$ obtained from $k = s_{\hat{r},\eps,\delta} \cdot \sqrt{\log (\hat{r}+1)}$ IID samples of $p^*$.
\textsc{TestAndMatch} produces a matching of size $m$ with competitive ratio of at least $\frac{\hat{n}}{n} - \frac{L_1(p,q)}{2} \geq \beta$ when $\hat{L}_1 \leq 2 \left( \frac{\hat{n}}{n} - \beta \right) - \eps$, and at least $\beta \cdot (1 - \frac{k}{n})$ otherwise, with success prob.\ $1-\delta$.
\end{theorem}

For sufficiently large $n$ and constants $\eps, \delta$, we have $s_{\hat{r},\eps,\delta} \cdot \sqrt{\log (\hat{r}+1)} \in o(1)$, so \cref{thm:main-result} implies a lower bound on the achieved competitive ratio of $\frac{m}{n^*}$ (see \cref{fig:ratio-plot}) where
\[
\frac{m}{n^*}
\geq \frac{m}{n}
\geq
\begin{cases}
\frac{\hat{n}}{n} - \frac{L_1(p^*, q)}{2} & \text{when $\hat{L}_1 \leq 2 \left( \frac{\hat{n}}{n} - \beta \right) - \eps$}\\
\beta \cdot (1 - o(1)) & \text{otherwise}
\end{cases}
\]

\begin{figure}[htb]
\centering
\resizebox{0.5\linewidth}{!}{
    \begin{tikzpicture}
\draw[thick, -stealth] (0,0) -- (0,4);
\draw[thick, -stealth] (0,0) -- (6,0);
\draw[dashed] (-0.25,3) node[left] {$\frac{\hat{n}}{n}$} -- (5,3);
\draw[dashed] (-0.25,1.5) node[left] {$\beta$} -- (5,1.5);
\draw[dashed] (-0.25,1) -- (5,1) node[right] {$\beta \cdot (1 - o_n(1))$};
\draw[dashed] (5,-0.25) node[below] {$1$} -- (5,3);
\draw[dashed] (2.5,-0.25) node[below] {$2 \left( \frac{\hat{n}}{n} - \beta \right) - 2 \eps$} -- (2.5,3);

\node[] at (-0.25,-0.25) {$0$};
\node[] at (6,-0.5) {$L_1(p^*, q)$};
\node[] at (2.6,3.75) {
\begin{tabular}{l}
Lower bound on achieved\\
competitive ratio (w.p.\ $\geq 1 - \delta$)
\end{tabular}};
\draw[ultra thick, blue] (0,3) -- (2.5,1.5);
\draw[ultra thick, blue] (2.5,1) -- (5,1);
\end{tikzpicture}
}
\caption{A (conservative) competitive ratio plot for $\frac{\hat{n}}{n} > \beta$. If \textsc{MinimaxTest} (\cref{alg:minimaxtest}) succeeds, we have $L_1(p^*, q) < 2 \left( \frac{\hat{n}}{n} - \beta \right) - 2 \eps$ whenever $\hat{L}_1 < 2 \left( \frac{\hat{n}}{n} - \beta \right) - \eps$. Observe that there is a smooth interpolation between the achieveable competitive ratio as $L_1(p^*, q)$ degrades whilst paying only $o(1)$ for robustness.}
\label{fig:ratio-plot}
\end{figure}
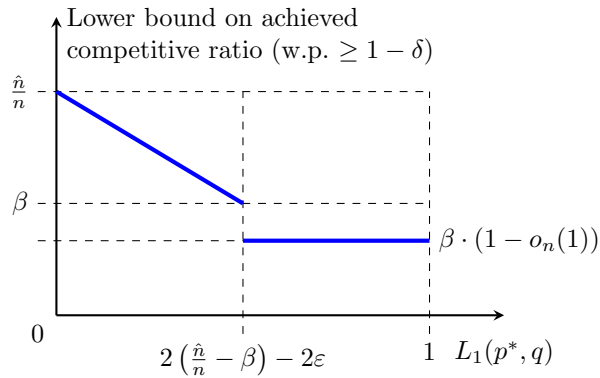

Under random order arrivals, the competitive ratio is measured in expectation over all possible arrival sequences.
One can easily convert the guarantees of \cref{thm:main-result} to one in expectation by assuming the extreme worst case scenario of obtaining 0 matches whenever the tester fails.
So, the expected competitive ratio is simply $(1 - \delta)$ factor of the bounds given in \cref{thm:main-result}.
Setting $\delta = 0.001$, we get a robustness guarantee of $\beta \cdot (1 - o(1)) \cdot 0.999$ in expectation.
Note that our guarantees hold regardless of what value of $\eps$ is used.
In the event that a very small $\eps$ is chosen and the test always fails, we are still guaranteed the robustness guarantees of $\approx \beta$.
One possible default for $\eps$ could be to assume that the optimal offline matching has size $n$ and just set it to half the threshold value, i.e. set $\eps = \hat{n}/n - \beta$.

\paragraph{Remark about lines 4 and 6 in \textsc{TestAndMatch}.}
As we subsequently require $\textrm{Poisson}(s_{\hat{r},\eps,\delta})$ IID samples from $p^*$ for testing, we collect $s_{\hat{r},\eps,\delta} \cdot \sqrt{\log (\hat{r}+1)}$ online arrivals into the set $A$.
Note that $\E[\textrm{Poisson}(s_{\hat{r},\eps,\delta})] = s_{\hat{r},\eps,\delta}$ and $\textrm{Poisson}(s_{\hat{r},\eps,\delta}) \leq s_{\hat{r},\eps,\delta} \cdot \sqrt{\log (\hat{r}+1)}$ with high probability.
This additional slack of $\sqrt{\log (\hat{r}+1)}$ allows for \cref{thm:main-result} to hold with high probability (as opposed to constant) while ensuring that the competitive ratio remains in the $\beta \cdot (1-o(1))$ regime.
Finally, when $r \in \Omega(n)$, we remark that $s_{\hat{r},\eps,\delta}$ is sublinear in $n$ only for sufficiently large $n$; see \cref{sec:extensions} for some practical modifications.

The rest of this section is devoted to describing \cref{alg:testandmatch} and proving \cref{thm:main-result}.
We study in \cref{sec:effect-of-advice-quality} how mimicking poor advice quality impacts matching sizes, yielding conditions where mimicking is desirable, which we test for via \cref{thm:minimax}.
\cref{sec:estimating-advice-quality} describes transformations to massage our problem into the form required by \cref{thm:minimax}.
Lastly, we tie up our analysis of \cref{thm:main-result} in \cref{sec:advice-property-testing}.

\subsection{Effect of advice quality on matching sizes}
\label{sec:effect-of-advice-quality}

\begin{algorithm}[t]
\caption{$\textsc{Mimic}$}
\label{alg:mimic-subroutine}
\begin{algorithmic}[1]
    \INPUT Matching $\hat{M}$, advice counts $\hat{c}$, arrival type label $t$
    \IF{$c(t) > 0$}
        \STATE Mimic an arbitrary unused type $t$ match in $\hat{M}$
        \STATE Decrement $c(t)$ by 1
    \ENDIF
    \STATE \textbf{return} $c$ \COMMENT{Updated counts}
\end{algorithmic}
\end{algorithm}

Given an advice $\hat{c}$ of type counts, we first solve optimally for a maximum matching $\hat{M}$ on the advice graph $\hat{\cG}$ and then mimic the matches for online arrivals whenever possible; see \cref{alg:mimic-subroutine}.
That is, whenever new vertices arrive, we match according to some unused vertex of the same type if possible and leave it unmatched otherwise.

It is useful to normalize counts as $p^* = c^*/n$ and $q = \hat{c}/{n}$. 
These are distributions on the realized and predicted (by advice) counts, and have sparse support $T^*$ and $\hat{T}$.

Now, suppose $\hat{M}$ has matching size $\hat{n} \leq n$.
By definition of $L_1(c^*, \hat{c})$ and \textsc{Mimic}, one would obtain a matching of size at least $\hat{n} - L_1(c^*, \hat{c}) / 2$ by blindly following advice.
This yields a competitive ratio of $\frac{\hat{n} - L_1(c^*, \hat{c}) / 2}{n^*}$.
Rearranging, we see that \textsc{Mimic} outperforms the advice-free baseline (in terms of worst case guarantees) if and only if
\begin{align*}
\frac{\hat{n} - L_1(c^*, \hat{c}) / 2}{n^*} > \beta
\iff &
L_1(p^*, q) < \frac{2}{n} \left( \hat{n} - \beta n^* \right)
\end{align*}

The above analysis suggests a natural way to use advice type counts: use \textsc{Mimic} if $L_1(p^*, q) \leq \frac{2}{n} \left( \hat{n} - \beta n^* \right)$, and \textsc{Baseline} otherwise.
Note that one should always just use \textsc{Baseline} whenever $\frac{\hat{n}}{n^\ast} < \beta$, matching the natural intuition of ignoring advice of poor quality.

Unfortunately, as we only know $n$ but not $n^*$, our algorithm conservatively checks whether $L_1(p^*, q) < 2 \left( \hat{n}/n - \beta \right)$, and so the resulting guarantee is \emph{conservative} since $n^* \leq n$.

\subsection{Estimating advice quality via property testing}
\label{sec:estimating-advice-quality}
As $c^*$ is unknown, we cannot obtain $L_1(p^*,q)$. However, $p^*$ and $q$ are distributions and we can apply the property testing method of \cref{thm:minimax} to estimate $L_1(p^*,q)$ to some $\eps > 0$ accuracy. 
Applying \cref{thm:minimax} raises two difficulties.

\paragraph{Simulating IID arrivals.} Under the Random Order arrival model (\cref{sec:prelims}), online vertices arrive ``without replacement'', which is incompatible with \cref{thm:minimax}.
Thankfully, we can apply a standard trick to simulate IID ``sampling with replacement'' from $p^*$ by ``re-observing arrivals''.
See \textsc{SimulateP} (\cref{alg:simulatep-subroutine}) in the appendix for details.

\paragraph{Operating in reduced domains.} Strictly speaking, the domain of $p^*$ and $q$ could be as large as $2^n$, since any one of these types may occur. If all of these types occur with non-zero probability, then applying \cref{thm:minimax} for testing could take a near-exponential (in $n$) number of online vertex arrivals, which is clearly impossible.
However, as established earlier, $p^*$ and $q$ enjoy sparsity; in particular, $\hat{c}$ and thus $q$ contain $0$ in all but at most $\hat{r} = |\hat{T}| \leq n \ll 2^n$ entries. 
The key insight is to express $L_1$ distances by operating on $\hat{T}$, plus an additional dummy type $t_0$ which has $0$ counts in $\hat{c}$. Whenever we observe an online vertex with type $t \in T^* \backslash \hat{T}$, we classify it as $t_0$. Specifically,
\[
L_1(p^*, q)
= \sum_{t \in 2^U} |p^*(t) - q(t)| 
= \sum_{t \in \hat{T} \cup T^*} |p^*(t) - q(t)|
= \sum_{t \in \hat{T}} |p^*(t) - q(t)| + \sum_{t \in T^* \backslash \hat{T}} p^*(t),
\]
which we can view as an $L_1$ distance on distributions with support $\hat{T} \cup \{t_0 \} $. Thus, the domain size when applying \cref{thm:minimax} is $\hat{r}+1 \leq n + 1$.
For any constant $\epsilon, \delta > 0$, the required samples is $s_{\hat{r}, \eps, \delta} \cdot \sqrt{\log (\hat{r}+1)} \in o(n)$.

Now that these difficulties are overcome, the estimation of $L_1(p^*,q) = L_1(\frac{c^*}{n}, \frac{\hat{c}}{n})$ is done via \textsc{MinimaxTest} (\cref{alg:minimaxtest}), whose correctness follows from \cref{thm:minimax}.

\begin{algorithm}[t]
\caption{$\textsc{MinimaxTest}(s, \frac{\hat{c}}{n}, A, \tau, \delta')$}
\label{alg:minimaxtest}
\begin{algorithmic}[1]
    \INPUT Sample size $s$, distribution $q=\hat{c}/n$, $s \sqrt{\log (\hat{r}+1)}$ online arrivals $A$, testing threshold $\tau$, failure rate $\delta'$
    \STATE Compute $s_1, s_2 \sim \mathrm{Poisson}(s/2)$
    \IF{$s_1 + s_2 > s \sqrt{\log (\hat{r}+1)}$}
        \STATE \textbf{return} ``Fail''
        \COMMENT{occurs w.p. $ \delta_{poi}\leq 1/poly(\hat{r})$}
    \ENDIF
    \STATE Collect $s_1 + s_2$ IID samples from $p^* = \frac{c^*}{n}$ by running \textsc{SimulateP} with $A$.
    \STATE Run algorithm of \cref{thm:minimax} to obtain estimate $\hat{L}_1$ such that $|\hat{L}_1 - L_1(p,q)| \leq \eps$ with probability $1-\delta'$
    \STATE \textbf{return} $\hat{L}_1 < \tau$
\end{algorithmic}
\end{algorithm}

\begin{lemma}\label{lem:minimax_error}
Suppose \textsc{MinimaxTest} uses the estimator of \cref{thm:minimax} and passes if and only if $\hat{L}_1 < \tau$.
Given $s_{\hat{r},\eps,\delta} \cdot \sqrt{\log (\hat{r}+1)}$ online arrivals in a set $A$, we have $L_1(p^*, q) < 2 \left( \hat{n}/n - \beta \right)$ whenever \textsc{MinimaxTest} passes.
The success probability of \textsc{MinimaxTest} is at least $1 - \delta$.
\end{lemma}
\begin{proof}
The algorithm of \cref{thm:minimax} guarantees tells us that $|\hat{L}_1 - L_1(p^*,q)| \leq \eps$ with probability at least $1 - \delta'$.
Therefore, when \textsc{MinimaxTest} passes, we are guaranteed that $L_1(p^*,q) \leq \hat{L}_1 + \eps < \tau + \eps = 2 \left( \hat{n}/n - \beta \right)$.

Meanwhile, in the analysis of \cref{thm:minimax}, one actually needs to use $s_1 + s_2$ IID samples from $p^*$, where $s_1, s_2 \sim \mathrm{Poisson}(s_{\hat{r},\eps,\delta'})$, which can be simulated from the arrival set $A$; see \textsc{SimulateP} (\cref{alg:simulatep-subroutine}) in the appendix.
By \cref{lem:poisson-bound}, we may assume that $s_1 + s_2 \leq s$ with probability at least $1-\delta_{poi}(s)$.
Taking a union bound over the failure probability of the ``Poissonization'' event and the estimator, we see that the overall success probability is at least $1-(\delta'+\delta_{poi})=1 - \delta$.
\end{proof}

\subsection{Tying up our analysis of \textsc{TestAndMatch}}
\label{sec:advice-property-testing}

If we run \textsc{Baseline} from the beginning due to $\frac{\hat{n}}{n} \leq \beta$, then we trivially recover a $\beta$-competitive ratio.
The following lemma gives a lower bound on the obtained matching size if we performed \textsc{MinimaxTest} but decided switch to \textsc{Baseline} due to the estimated $\hat{L}_1$ being too large.

\begin{lemma}
\label{lem:run-on-after-k-steps}
Suppose we run an arbitrary algorithm for the first $k \in [n]$ online arrivals and then switch to $\textsc{Baseline}$ for the remaining $n-k$ online arrivals.
If $j$ matches made in the first $k$ arrivals, where $0 \leq j \leq k$, then the overall produced matching size is at least $\beta \cdot (n - k - j) + j$.
\end{lemma}
\begin{proof}
Any match made in the first $k$ arrivals decreases the maximum attainable matching size by at most two, \emph{excluding the match made}.
As the maximum attainable matching size was originally $n$, the maximum attainable matching size on the postfix sequence after the $k$ is at least $n - k - j$.
Since $\textsc{Baseline}$ has competitive ratio $\beta$, running $\textsc{Baseline}$ on the remaining $n-k$ steps will produce a matching of size at least $\beta \cdot (n - k - j)$.
Thus, the overall produced matching size is at least $\beta \cdot (n - k - j) + j$.
\end{proof}

The proof of \cref{thm:main-result} requires the following lemma.

\begin{lemma}
\label{lm:main-result}
For any advice $\hat{c}$ with $|\hat{T}| = \hat{r}$, $\eps > 0$ and $\delta > \frac{1}{poly(\hat{r})}$, let $\hat{L}_1$ be the estimate of $L_1(p^*,q)$ in \textsc{MinimaxTest}.
If \textsc{MinimaxTest} succeeds, then \textsc{TestAndMatch} produces a matching of size $m$ with competitive ratio $\frac{m}{n^*}$ at least $\frac{\hat{n}}{n} - \frac{L_1(p,q)}{2} \geq \beta$ when $\hat{L}_1 \leq 2 \left( \frac{\hat{n}}{n} - \beta \right) - \eps$, and at least $\beta \cdot (1 - \frac{s_{\hat{r},\eps,\delta} \cdot \sqrt{\log (\hat{r}+1)}}{n})$ otherwise.
\end{lemma}
\begin{proof}
We consider each case separately.

\textbf{Case 1}: $\hat{L}_1 \leq 2 \left( \frac{\hat{n}}{n} - \beta \right) - \eps$\\
\textsc{TestAndMatch} executed \textsc{Mimic} for all online arrivals, yielding a matching of size $m \geq \hat{n} - \frac{L_1(c^*, \hat{c})}{2}$.
Since \textsc{MinimaxTest} succeeds, $|\hat{L}_1 - L_1(p,q)| \leq \eps$, so
$
L_1(p,q)
\leq \hat{L}_1 + \eps
\leq 2 \left( \frac{\hat{n}}{n} - \beta \right) - \eps + \eps
= 2 \left( \frac{\hat{n}}{n} - \beta \right)
$.
Therefore,
\[
\frac{m}{n^*}
\geq \frac{m}{n}
\geq \frac{\hat{n}}{n} - \frac{L_1(c^*, \hat{c})}{2 n}
= \frac{\hat{n}}{n} - \frac{L_1(p^*, q)}{2}
\geq \beta
\]

\textbf{Case 2}: $\hat{L}_1 > 2 \left( \frac{\hat{n}}{n} - \beta \right) - \eps$\\
\textsc{TestAndMatch} executes \textsc{Baseline} after an initial batch of $k = s_{\hat{r},\eps,\delta} \cdot \sqrt{\log (\hat{r}+1)}$ arrivals that follow \textsc{Mimic}.
Suppose we made $j$ matches via \textsc{Mimic} before \textsc{MinimaxTest}.
Then, \cref{lem:run-on-after-k-steps} tells us that the overall produced matching size is at least $m \geq \beta \cdot (n - k - j) + j$.
Since $\beta < 1$, we have
$
\beta \cdot (n - k - j) + j
\geq \beta \cdot (n - k)
$.
Therefore,
\[
\frac{m}{n^*}
\geq \frac{m}{n}
\geq \frac{\beta \cdot (n - s_{\hat{r},\eps,\delta} \cdot \sqrt{\log (\hat{r}+1)})}{n}
= \beta \cdot \left( 1 - \frac{s_{\hat{r},\eps,\delta} \cdot \sqrt{\log (\hat{r}+1)}}{n} \right)
\qedhere
\]
\end{proof}

\cref{thm:main-result} follows from bounding the failure probability.

\begin{proof}[Proof of \cref{thm:main-result}]
The competitive ratio guarantees follow directly from \cref{lm:main-result}, given that \textsc{MinimaxTest} succeds.
Therefore, it only remains to bound the failure probability, which equals that probability that \textsc{MinimaxTest} fails.
This can happen if either line $3$ is executed (event $E_1$) or the algorithm in line $5$ fails (event $E_2$).

The event $E_1$ occurs when the one of the Poisson random variables in line $1$ of Algorithm \ref{alg:minimaxtest} exceed the expectation by a $\sqrt{\log \hat{r}}$ factor. Since $s_1,s_2\sim \mathrm{Poisson}(s/2)$, we have that $(s_1+s_2)\sim \mathrm{Poisson}(s)$. Thus, by Lemma \ref{lem:poisson-bound} we have that:
\[
\delta_{poi}
= \Pr[\vert (s_1+s_2) - s\vert > s\sqrt{\log \hat{r}} ]\\
\leq 2e^{-\frac{s^2\log \hat{r}}{2(s+s\sqrt{\log \hat{r}})}}
= O\left(\hat{r}^{-\frac{s}{2(1+\sqrt{\log \hat{r}} )}}\right)
= O\left( \frac{1}{poly(\hat{r})}\right)
\]
for the value of $s=\cO \left( \frac{(\hat{r}+1) \cdot \log(1/\delta')}{\eps^2 \cdot \log (\hat{r}+1)} \right)$ chosen.

Combining the above with \cref{lem:minimax_error} via union bound yields
$\Pr(\text{failure})
\leq \Pr(E_1) + \Pr(E_2)
\leq \delta_{poi} + \delta'
= \delta$.  
\end{proof}

\section{Practical considerations}
\label{sec:extensions}

While our contributions are mostly theoretical, we discuss some practical considerations here.
In particular, we would like to highlight that there is no existing practical implementation of the algorithm of \cref{thm:minimax} by \citet{JHW18}.
As is the case for most state-of-the-art distribution testing algorithms, this implementation is highly non-trivial and requires the use of optimal polynomial approximations over functions, amongst other complicated constructions.\footnote{The tester proposed by \citet{JHW18} requires a significant amount of hyperparameter tuning and no off-the-shelf implementation is available \cite{JHW-email}; see \cref{sec:appendix-poc} for more comments.}
For completeness, we implemented a proof-of-concept based on the empirical $L_1$ estimation; see \cref{sec:appendix-poc}.
While it is known that the estimation error scales with the sample size in the form $\Omega(r/\eps^2)$, we observe good empirical performance when $r$ is sublinear in $n$ or when combined with some of the practical extensions that we discussed below.

\cref{sec:sigma-remapping} and \cref{sec:coarsening-advice} can be viewed as ways to extend the usefulness of a given advice.
\cref{sec:patch-advice-till-perfect-matching} provides a way to ``patch'' an advice with $\hat{n} < n$ to one with perfect matching, without hurting the provable guarantees.
\cref{sec:true-support-size-is-small} gives a pre-processing step that can be prepended to any procedure: by losing $o(1)$, one can test whether $|T^*|$ is small and if so learn $p^*$ up to $\eps$ error to fully exploit it.

\subsection{Remapping online arrival types}
\label{sec:sigma-remapping}

Consider the graph example in \cref{fig:type-example} with type counts $c^*$ and we are given some advice count $\hat{c}$ as follows:

\begin{center}
\begin{tabular}{cc|cc}
\toprule
Types $T^*$ & $c^*$ count & Types $\hat{T}$ & $\hat{c}$ count\\
\midrule
$\{u_1, u_3\}$ & 1 & $\{u_1\}$ & 1\\
$\{u_2, u_3\}$ & 1 & $\{u_3\}$ & 1\\
$\{u_1, u_2, u_4\}$ & 2 & $\{u_4\}$ & 1\\
&& $\{u_2, u_4\}$ & 1\\
\bottomrule
\end{tabular}
\end{center}

While one can verify that both the true graph $\cG^*$ and the advice graph $\hat{\cG}$ have perfect matching, $L_1(c^*,\hat{c}) = 4$ since as $T^*$ and $\hat{T}$ have disjoint types.
From the perspective of our earlier analysis, $\hat{c}$ would be deemed as a poor quality advice and one should default to \textsc{Baseline}.

However, a closer look reveals there exists a mapping $\sigma$ from $T^*$ to $\hat{T}$ such that one can credibly ``mimic'' the proposed matching of $\hat{\cG}$ as online vertices arrive.
For example, when an online vertex $v$ with neighborhood type $\{u_1, u_3\}$ arrive, one can ``ignore'' the edge $u_3 \sim v$ and treat it as if $v$ had the type $\{u_1\}$.
Similarly, $\{u_2, u_3\}$ could be treated as $\{u_3\}$, the first instance of $\{u_1, u_2, u_4\}$ could be treated as $\{u_2, u_4\}$, and the second instance of $\{u_1, u_2, u_4\}$ could be treated as $\{u_4\}$.
Running \textsc{Mimic} under such a remapping of online types would then produce a perfect matching!
We discuss how to perform such remappings in \cref{sec:appendix-remapping-online-arrival-types}.

\subsection{Coarsening of advice}
\label{sec:coarsening-advice}

While \cref{thm:main-result} has good asymptotic guarantees as $n \to \infty$, the actual number of vertices $n$ is finite in practice.
In particular, when $n$ is ``not large enough'', \textsc{TestAndMatch} will never utilize the advice and always default to \textsc{Baseline} for all problem instances where $n \ll s_{\hat{r},\eps,\delta}$.

In practice, while the given advice types may be diverse, there could be many ``overlapping subtypes'' and a natural idea is to ``coarsen'' the advice by grouping similar types together in an effort to reduce the resultant support size of the advice (and hence $s_{\hat{r},\eps,\delta}$).
\cref{fig:coarsening-example} illustrates an extreme example where we could decrease the support size from $n$ to $2$ while still maintaining a perfect matching.
In \cref{sec:appendix-advice-coarsening}, we explain two possible ways to coarsen $\hat{c}$.

\begin{figure}[htb]
\centering
\resizebox{0.5\linewidth}{!}{
    \begin{tikzpicture}[square/.style={regular polygon,regular polygon sides=4}]
\node[] at (1,0.5) {$\hat{\cG}$};

\node[draw, thick, square, minimum size=10pt, inner sep=1pt] at (0,0) (u1) {};
\node[left=1pt of u1] {$u_1$};
\node[] at (0,-1) (udots) {$\vdots$};
\node[draw, thick, square, minimum size=10pt, inner sep=1pt] at (0,-2) (uhalf) {};
\node[left=1pt of uhalf] {$u_{\frac{n}{2}}$};
\node[draw, thick, square, minimum size=10pt, inner sep=1pt] at (0,-3) (uhalfplusone) {};
\node[] at (0,-4) (udots2) {$\vdots$};
\node[left=1pt of uhalfplusone] {$u_{\frac{n}{2}+1}$};
\node[draw, thick, square, minimum size=10pt, inner sep=1pt] at (0,-5) (un) {};
\node[left=1pt of un] {$u_n$};

\node[draw, thick, circle, minimum size=10pt, inner sep=1pt] at (2,0) (v1) {};
\node[right=1pt of v1] {$v_1$};
\node[] at (2,-1) (vdots) {$\vdots$};
\node[draw, thick, circle, minimum size=10pt, inner sep=1pt] at (2,-2) (vhalf) {};
\node[right=1pt of vhalf] {$u_{\frac{n}{2}}$};
\node[draw, thick, circle, minimum size=10pt, inner sep=1pt] at (2,-3) (vhalfplusone) {};
\node[] at (2,-4) (vdots2) {$\vdots$};
\node[right=1pt of vhalfplusone] {$v_{\frac{n}{2}+1}$};
\node[draw, thick, circle, minimum size=10pt, inner sep=1pt] at (2,-5) (vn) {};
\node[right=1pt of vn] {$v_n$};

\draw[thick] (v1) -- (u1);
\draw[thick] (v1) -- (udots);
\draw[thick] (v1) -- (uhalf);
\draw[thick] (vdots) -- (u1);
\draw[thick] (vdots) -- (udots);
\draw[thick] (vdots) -- (uhalf);
\draw[thick] (vhalf) -- (u1);
\draw[thick] (vhalf) -- (udots);
\draw[thick] (vhalf) -- (uhalf);
\draw[thick] (vhalfplusone) -- (uhalfplusone);
\draw[thick] (vhalfplusone) -- (udots2);
\draw[thick] (vhalfplusone) -- (un);
\draw[thick] (vdots2) -- (uhalfplusone);
\draw[thick] (vdots2) -- (udots2);
\draw[thick] (vdots2) -- (un);
\draw[thick] (vn) -- (uhalfplusone);
\draw[thick] (vn) -- (udots2);
\draw[thick] (vn) -- (un);
\draw[thick, dashed, red] (v1) -- (uhalfplusone);
\draw[thick, dashed, red] (vdots) -- (udots2);
\draw[thick, dashed, red] (vhalf) -- (un);
\draw[thick, dashed, red] (vhalfplusone) -- (u1);
\draw[thick, dashed, red] (vdots2) -- (udots);
\draw[thick, dashed, red] (vn) -- (uhalf);

\node[draw=blue, single arrow, single arrow head extend=5pt, thick, minimum height=20pt, inner sep=1pt] at (3.5,-2.5) {};

\node[] at (6,0.5) {$\hat{\cG}'$};

\node[draw, thick, square, minimum size=10pt, inner sep=1pt] at (5,0) (cu1) {};
\node[left=1pt of cu1] {$u_1$};
\node[] at (5,-1) (cudots) {$\vdots$};
\node[draw, thick, square, minimum size=10pt, inner sep=1pt] at (5,-2) (cuhalf) {};
\node[left=1pt of cuhalf] {$u_{\frac{n}{2}}$};
\node[draw, thick, square, minimum size=10pt, inner sep=1pt] at (5,-3) (cuhalfplusone) {};
\node[] at (5,-4) (cudots2) {$\vdots$};
\node[left=1pt of cuhalfplusone] {$u_{\frac{n}{2}+1}$};
\node[draw, thick, square, minimum size=10pt, inner sep=1pt] at (5,-5) (cun) {};
\node[left=1pt of cun] {$u_n$};

\node[draw, thick, circle, minimum size=10pt, inner sep=1pt] at (7,0) (cv1) {};
\node[right=1pt of cv1] {$v_1$};
\node[] at (7,-1) (cvdots) {$\vdots$};
\node[draw, thick, circle, minimum size=10pt, inner sep=1pt] at (7,-2) (cvhalf) {};
\node[right=1pt of cvhalf] {$u_{\frac{n}{2}}$};
\node[draw, thick, circle, minimum size=10pt, inner sep=1pt] at (7,-3) (cvhalfplusone) {};
\node[] at (7,-4) (cvdots2) {$\vdots$};
\node[right=1pt of cvhalfplusone] {$v_{\frac{n}{2}+1}$};
\node[draw, thick, circle, minimum size=10pt, inner sep=1pt] at (7,-5) (cvn) {};
\node[right=1pt of cvn] {$v_n$};

\draw[thick] (cv1) -- (cu1);
\draw[thick] (cv1) -- (cudots);
\draw[thick] (cv1) -- (cuhalf);
\draw[thick] (cvdots) -- (cu1);
\draw[thick] (cvdots) -- (cudots);
\draw[thick] (cvdots) -- (cuhalf);
\draw[thick] (cvhalf) -- (cu1);
\draw[thick] (cvhalf) -- (cudots);
\draw[thick] (cvhalf) -- (cuhalf);
\draw[thick] (cvhalfplusone) -- (cuhalfplusone);
\draw[thick] (cvhalfplusone) -- (cudots2);
\draw[thick] (cvhalfplusone) -- (cun);
\draw[thick] (cvdots2) -- (cuhalfplusone);
\draw[thick] (cvdots2) -- (cudots2);
\draw[thick] (cvdots2) -- (cun);
\draw[thick] (cvn) -- (cuhalfplusone);
\draw[thick] (cvn) -- (cudots2);
\draw[thick] (cvn) -- (cun);
\end{tikzpicture}
}
\caption{Consider $\hat{\cG}$ made by taking the union of two complete bipartite graphs ($\hat{\cG}'$) and adding the red dashed edges. By connecting $v_i$ to $u_{(i + n/2) \mathrm{mod\ } n}$, $|\hat{T}| = r = n$. Meanwhile, if we coarsen $\hat{c}$ into $\hat{c}'$ by ignoring the red dashed edges, $\hat{\cG}'$ still has a maximum matching of size $\hat{n}' = n$ while $|\hat{T}'| = r' = 2$, thus requiring significantly less samples to test since $s_{\hat{r}',\eps,\delta} \ll s_{\hat{r},\eps,\delta}$. Furthermore, if $\cG^* = \hat{G}'$, then $L_1(c^*, \hat{c}) = 2n$ and we will reject the advice $\hat{c}$ if we do not coarsen it first.}
\label{fig:coarsening-example}
\end{figure}
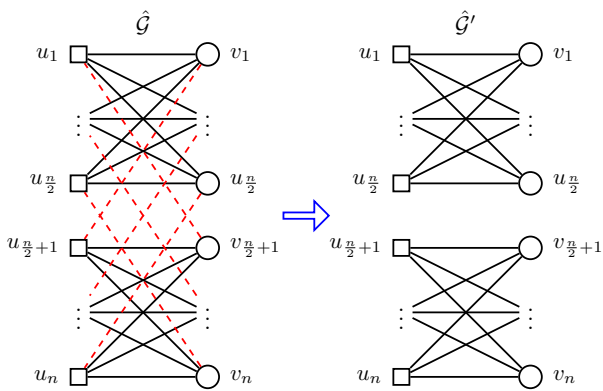

\subsection{Advice does not have perfect matching}
\label{sec:patch-advice-till-perfect-matching}

As the given advice $\hat{c}$ is arbitrary, it could be the case that any maximum matching of size $\hat{n}$ in the graph $\hat{\cG}$ implied by $\hat{c}$ is not perfect, i.e.\ $\hat{n} < n$.
A natural idea would be to ``patch'' $\hat{c}$ into some other type count $\hat{c}'$ which has a maximum matching size of $\hat{n}' = n$ in the tweaked graph $\hat{\cG}'$.
This can be done by augmenting $\hat{c}$ with additional edges between the unmatched vertices in the advice graph to obtain $\hat{c}'$.
In \cref{sec:appendix-augmenting-to-perfect-matching}, we show how to do this in a way that running \textsc{TestAndMatch} on $\hat{c}'$ does not worsen the provable guarantees as compared to directly using $\hat{c}$.

\subsection{True distribution has small support size}
\label{sec:true-support-size-is-small}

If the support size of the true types is $o(n)$, a natural thing to do is to \emph{learn} $c^*$ up to some $\eps$ accuracy while forgoing some $o(n)$ initial matches, and then obtain $\approx 1-\eps$ competitive ratio on the remaining arrivals.
Though this is wholly possible in the random arrival model, it crucially depends on $c^*$ having at most $o(n)$ types.
Although we do not know the support size of $c^*$ a priori, we can again employ techniques from property testing.
For any desired support size $k$ and constant $\eps$, \citet{valiant2017estimating,wu2019chebyshev} tell us that $O(\frac{k}{\log k})$ samples are sufficient for us to estimate the support size of a discrete distribution up to additive error of $\eps k$.
Therefore, for any $k \in o(n)$ and constant $\eps$, given any algorithm ALG under the random arrival model achieving competitive ratio $\alpha$, we can first spend $o(1)$ arrivals to test whether $c^*$ is supported on $(1 + \eps) \cdot k$ types:
\begin{itemize}
    \item If ``Yes'', then we can spend another $O(k/\eps^2) \subseteq o(1)$ arrivals to estimate $c^*$ up to $\eps$ accuracy, i.e.\ we can form $\hat{c}$ with $L_1(c^*, \hat{c}) \leq \eps$, then exploit $\hat{c}$ via \textsc{Mimic}.
    \item If ``No'', use ALG and achieve a comp.\ ratio of $\alpha - o(1)$.
\end{itemize}
The choice of $k$ is flexible in practice, depending on how much one is willing to lose in the $o(1)$ in the ``No'' case.

\section{Discussion and future directions}
\label{sec:conclusion}

We studied the online bipartite matching problem with respect to \cref{goal:test}.
We showed that it is impossible under the adversarial arrival model and designed a meta algorithm \textsc{TestAndMatch} for the random arrival model that is 1-consistent and $\beta \cdot (1 - o(1))$-robust while using histograms over arrival types as advice.
The guarantees \textsc{TestAndMatch} degrades gracefully as the quality of the advice worsens, and improves whenever the state-of-the-art $\beta$ improves. There are several interesting follow-up questions:

\begin{enumerate}
    \item \textbf{Other versions of online bipartite matching.}\\
    Whether the ideas presented in this work can be generalized to other versions of the online bipartite matching problem is indeed an interesting question.
    The hardness results of \cref{thm:hardness} directly translate as the unit weight version is a special case of the general case with vertex or edge weights, implying impossibility more general settings.
    Algorithmically, we believe that the \textsc{TestAndMatch} framework should generalize.
    However, it would require a different advice quality metric in the space of edge weights (e.g., something like earthmover distance), along with a corresponding sample efficient way to test this quality with sublinear samples so that one can still achieve a competitive ratio of $\beta \cdot (1 - o(1))$ when the test detects that the advice quality is bad.

    \item \textbf{Consistency and robustness tradeoff.}\\
    Our algorithm is a meta-algorithm that is both 1-consistent and $\beta \cdot (1 - o(1))$-robust, where the robustness guarantee improves with the state-of-the-art (with respect to $\beta$).
    As, it is impossible for any algorithm to be strictly better than 1-consistent (by definition of competitive ratio) or strictly better than $\beta$-robust (by definition of $\beta$), our algorithm (weakly-)pareto-dominates all other possible algorithm for this problem up to an $1 - o(1)$ multiplicative factor in the robustness guarantee.
    Can this $1 - o(1)$ multiplicative factor be removed?

    \item \textbf{Beyond consistency and robustness.}\\
    \textsc{TestAndMatch}'s performance guarantee is based on the $L_1$ distance over type histograms. This is very sensitive to certain types of noise, e.g., adding or removing edges at random (Erdos-Renyi). However, \cref{sec:extensions} suggests there are practical extensions that hold even when $L_1$ is large, implying it is a non-ideal metric despite satisfying consistency and robustness. Is there another criterion that could fill this gap?

    \item \textbf{Going beyond \textsc{Mimic}.}\\
    Our current approach exploits advice solely through \textsc{Mimic}, which arbitrarily chooses \emph{one} matching $\hat{M}$ to follow. Is there a more intelligent way of doing so? For example, \citet{feldman2009online} constructed \emph{two} matchings to ``load balance'' in the known IID setting.

    \item \textbf{Is there a graph where advice coarsening recovers perfect matching while \textsc{Ranking} does not have a competitive ratio close to 1?}\\
    \textsc{Ranking} is known to have a competitive ratio of $1-1/\sqrt{k}$ if there exist $k$ disjoint perfect matchings in the graph \citep{karande2011online}.
    Beyond trivial settings where there are no perfect matchings or having a pattern connecting to just 1 offline vertex (so that there is at most 1 disjoint perfect matching), we do not have an example with a formal proof that advice coarsening recovers a perfect matching while \textsc{Ranking} does not have a competitive ratio close to 1.
    We suspect that one may not be able to construct a graph with few disjoint matchings and few patterns.
    To see why, fix some graph with a perfect matching and suppose there are $q$ types $t_1, t_2, \ldots, t_q$ of sizes $s_1, \ldots, s_q$.
    By ``circularly permuting the matches'', we see that there will be $k = \min(s_1, \ldots, s_q)$ disjoint perfect matchings.
    However, this does not totally invalidate \textsc{TestAndMatch} in general.
    For example, consider the following augmentation to our approach: if $\hat{c}$ has $k$ disjoint perfect matchings after coarsening, one can run \textsc{Ranking} instead of \textsc{Mimic} and switch to \textsc{Baseline} if the test fails later.
    By doing so, one pays only an additive $1/\sqrt{k}$ in consistency by running \textsc{Ranking} instead of \textsc{Mimic} in the event that the advice was perfect.
    Note that \textsc{Baseline} may be different from \textsc{Ranking} as it is still unknown what is the true $\beta \in [0.696, 0.823]$.
    In particular, if $\beta > 0.727$, then \textsc{Baseline} cannot be \textsc{Ranking} \citep{karande2011online}.

    \item \textbf{Other online problems with random arrivals.}\\
    \textsc{TestAndMatch} did not exploit any specific properties of bipartite matching, and we suspect it may be generalized to a certain class of online problems.
\end{enumerate}

\section*{Acknowledgements}
This research/project is supported by the National Research Foundation, Singapore under its AI Singapore Programme (AISG Award No: AISG-PhD/2021-08-013).
AB and TG were supported by National Research Foundation Singapore under its NRF Fellowship Programme (NRF-NRFFAI1-2019-0002).
We would like to thank the reviewers for valuable feedback and discussions.
DC and CKL thank Anupam Gupta for discussions about learning-augmented algorithms.
DC would also like to thank Clément Canonne and Yanjun Han for discussions about distribution testing, and Yongho Shin about discussions about learning-augmented algorithms.

\section*{Impact statement}
While the contributions of this paper is mostly theoretical, one should be wary of possible societal impacts if online matching algorithms are implemented in practice.
For instance, our proposed methods did not explicitly account for possible fairness issues and such concerns warrant further investigation before being operationalized in real-world settings.

\bibliography{refs}
\bibliographystyle{plainnat}

\newpage
\appendix
\section{Extended variant of \texorpdfstring{\cref{thm:hardness}}{Theorem 3.1}}
\label{sec:appendix-extension-hardness}

Let us first prove the case when the algorithm $\cA$ is deterministic, but $\alpha \in [0, 1/2]$.
We will again use $\cG_1$ and $\cG_2$ of \cref{fig:hardness} (replicated below for convenience as \cref{fig:hardness-restated}) as a counterexample. 
Our argument follows that of the case where $\alpha=0$.

\begin{figure}[h]
\centering
\resizebox{0.7\linewidth}{!}{
    \begin{tikzpicture}[square/.style={regular polygon,regular polygon sides=4}]
\tikzmath{\graphgap=0.75;}
\node[] at (1,0.75*\graphgap) {$\cG_1$};

\node[draw, thick, square, minimum size=10pt, inner sep=1pt] at (0,0) (u1) {};
\node[left=1pt of u1] {$u_1$};
\node[] at (0,-1*\graphgap) (udots) {$\vdots$};
\node[draw, thick, square, minimum size=10pt, inner sep=1pt] at (0,-2*\graphgap) (uhalf) {};
\node[left=1pt of uhalf] {$u_{\frac{n}{2}}$};
\node[draw, thick, square, minimum size=10pt, inner sep=1pt] at (0,-3*\graphgap) (uhalfplusone) {};
\node[] at (0,-4*\graphgap) (udots2) {$\vdots$};
\node[left=1pt of uhalfplusone] {$u_{\frac{n}{2}+1}$};
\node[draw, thick, square, minimum size=10pt, inner sep=1pt] at (0,-5*\graphgap) (un) {};
\node[left=1pt of un] {$u_n$};

\node[draw, thick, circle, minimum size=10pt, inner sep=1pt] at (2,0) (v1) {};
\node[right=1pt of v1] {$v_1$};
\node[] at (2,-1*\graphgap) (vdots) {$\vdots$};
\node[draw, thick, circle, minimum size=10pt, inner sep=1pt] at (2,-2*\graphgap) (vhalf) {};
\node[right=1pt of vhalf] {$u_{\frac{n}{2}}$};
\node[draw, thick, circle, minimum size=10pt, inner sep=1pt] at (2,-3*\graphgap) (vhalfplusone) {};
\node[] at (2,-4*\graphgap) (vdots2) {$\vdots$};
\node[right=1pt of vhalfplusone] {$v_{\frac{n}{2}+1}$};
\node[draw, thick, circle, minimum size=10pt, inner sep=1pt] at (2,-5*\graphgap) (vn) {};
\node[right=1pt of vn] {$v_n$};

\draw[thick] (v1) -- (u1);
\draw[thick] (v1) -- (uhalfplusone);
\draw[thick] (vdots) -- (udots);
\draw[thick] (vdots) -- (udots2);
\draw[thick] (vhalf) -- (uhalf);
\draw[thick] (vhalf) -- (un);
\draw[thick, red] (vhalfplusone) -- (u1);
\draw[thick, red] (vdots2) -- (udots);
\draw[thick, red] (vn) -- (uhalf);

\draw[thick, dashed] (-0.5,-2.5*\graphgap) -- (7.5,-2.5*\graphgap);

\node[] at (6,0.75*\graphgap) {$\cG_2$};

\node[draw, thick, square, minimum size=10pt, inner sep=1pt] at (5,0) (cu1) {};
\node[left=1pt of cu1] {$u_1$};
\node[] at (5,-1*\graphgap) (cudots) {$\vdots$};
\node[draw, thick, square, minimum size=10pt, inner sep=1pt] at (5,-2*\graphgap) (cuhalf) {};
\node[left=1pt of cuhalf] {$u_{\frac{n}{2}}$};
\node[draw, thick, square, minimum size=10pt, inner sep=1pt] at (5,-3*\graphgap) (cuhalfplusone) {};
\node[] at (5,-4*\graphgap) (cudots2) {$\vdots$};
\node[left=1pt of cuhalfplusone] {$u_{\frac{n}{2}+1}$};
\node[draw, thick, square, minimum size=10pt, inner sep=1pt] at (5,-5*\graphgap) (cun) {};
\node[left=1pt of cun] {$u_n$};

\node[draw, thick, circle, minimum size=10pt, inner sep=1pt] at (7,0) (cv1) {};
\node[right=1pt of cv1] {$v_1$};
\node[] at (7,-1*\graphgap) (cvdots) {$\vdots$};
\node[draw, thick, circle, minimum size=10pt, inner sep=1pt] at (7,-2*\graphgap) (cvhalf) {};
\node[right=1pt of cvhalf] {$u_{\frac{n}{2}}$};
\node[draw, thick, circle, minimum size=10pt, inner sep=1pt] at (7,-3*\graphgap) (cvhalfplusone) {};
\node[] at (7,-4*\graphgap) (cvdots2) {$\vdots$};
\node[right=1pt of cvhalfplusone] {$v_{\frac{n}{2}+1}$};
\node[draw, thick, circle, minimum size=10pt, inner sep=1pt] at (7,-5*\graphgap) (cvn) {};
\node[right=1pt of cvn] {$v_n$};

\draw[thick] (cv1) -- (cu1);
\draw[thick] (cv1) -- (cuhalfplusone);
\draw[thick] (cvdots) -- (cudots);
\draw[thick] (cvdots) -- (cudots2);
\draw[thick] (cvhalf) -- (cuhalf);
\draw[thick] (cvhalf) -- (cun);
\draw[thick, red] (cvhalfplusone) -- (cuhalfplusone);
\draw[thick, red] (cvdots2) -- (cudots2);
\draw[thick, red] (cvn) -- (cun);
\end{tikzpicture}
}
\caption{(Restated) Illustration of $\cG_1$ and $\cG_2$ for \cref{thm:hardness}}
\label{fig:hardness-restated}
\end{figure}
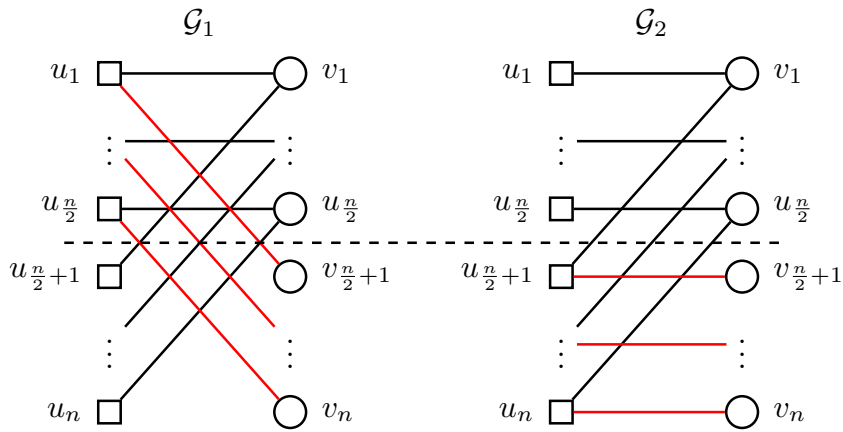

\textbf{Special case: $\cA$ is deterministic.}
As before, we observe that any algorithm cannot distinguish between the $\cG_1$ and $\cG_2$ after the first $n/2$ arrivals.
Suppose $\cA$ is $(1-\alpha)$-consistent.
Without loss of generality, by symmetry of the argument, suppose $\cG^* = \cG^2$ and $\cA$ is given advice bit $\hat{i} = 2$.

Since $\cA$ is $(1-\alpha)$-consistent, it has to make at least $\frac{n}{2} - (1 - \alpha) \cdot n$ matches in the first $\frac{n}{2}$ arrivals\footnote{Otherwise, even if the remaining $\frac{n}{2}$ vertices are matched, $\cA$ cannot achieve $(1-\alpha) \cdot n$ total matches, violating $(1 - \alpha)$-consistency}, leaving at most $\alpha \cdot n$ unmatched offline vertices amongst $\{ u_1, \ldots u_{\frac{n}{2}} \}$.
Meanwhile, if $\cG^* = \cG^1$ instead, there can only be at most ${\color{blue}\alpha \cdot n}$ matches amongst the remaining $\frac{n}{2}$ arrivals $\{v_{\frac{n}{2}+1}, \dots, v_n\}$, resulting in a total matching size of at most $\frac{n}{2} + {\color{blue}\alpha \cdot n} = \left( \frac{1}{2} + \alpha \right) \cdot n$.
That is, any deterministic $\cA$ that is $(1-\alpha)$-consistent cannot be strictly more than $\left( \frac{1}{2} + \alpha \right)$-robust.

\textbf{General case where $\cA$ could be randomized.}
Unfortunately, randomization does not appear to help much, as we can repeat all of the above arguments in expectation.
That is, if $\hat{i}=2$, it follows from consistency that in expectation, at least $(1-\alpha)\cdot n$ of all vertices must be eventually matched, meaning that in expectation there must be $\frac{n}{2}-\alpha\cdot n$ matches in the first half.
Now, if $\cG_1$ was the true graph, then in expectation we only have $\alpha\cdot n$ possible matches to make in the second half, thus we have a maximum of $\left( \frac{1}{2} + \alpha \right) \cdot n$ matches in expectation when $\hat{i}$ is wrong.
\section{Examples of realized type counts as advice}
\label{sec:appendix-examples-realized-counts-advice}

\paragraph{Example 1: Online Ads.} 
The canonical example of online bipartite matching is that of online ads \citep{mehta2013online}. Recall that the online vertices are advertisement slots (also called impressions) and the offline vertices are advertisers. We can see that the distribution over types can be possibly forecasted by machine learning models (and in fact, indirectly used \citep{alomrani2021deep} for bipartite matching) and used as advice. This directly gives us $q$, possibly bypassing $\hat{c}$. Regardless, the more accurate the forecasting, the lower $L_1(p^*, q)$ will be.

\paragraph{Example 2: Food allocation.}
Consider a conference organizer catering lunch.
As a cost-cutting measure, they cater \emph{exactly} one food item per attendee, based on their self-reported initial dietry preferences reported during registration (each attendee may report more than one item). 
During the conference, attendees will queue up in random order, \emph{sequentially} reporting their preferences once again and being assigned their food.
Organizers have the flexibility to assign food items based on this new reporting of preferences (or, in a somewhat morally questionable fashion refuse to serve the attendee---though in the unweighted setting, reasonable algorithms should not have to do this!).
Alas, a fraction of attendees claim a different preference from their initial preference, e.g., because they were fickle, or did not take initial dietry preference questionnaire seriously. 
Given that food is already catered, how should the conference organizers sequentially distribute meals to minimize hungry attendees?

The attendees are represented by $n$ online vertices, while each of the $n$ offline vertices represent one of $k$ types of food item\footnote{For practical settings, the types of food items is generally much smaller thatn $n$.}.
The attendees' initial preference gives our advice $q$ (the distribution over types of food prefernces), which also describes a perfect matching. This preference may differ from the distribution over true preferences reported on the day of the conference $p$. However, one can reasonable assume that only a small fraction of attendees exhibit such a mismatch, meaning that the $L_1(p^*,q)$ is fairly small and advice should be accepted most of the time.

\paragraph{Example 3: Centralized labor Allocation.}
Suppose there are $n$ employees and $m$ jobs. There are $\eta$ different qualifications. This is represented by a binary matrix $\{ 0, 1 \}^{n, \eta}$, where $X_{i, k} = 1$, if employee $i$ posseses qualification $k$. Therefore, the $i$-th row of $X$, $X_i$ is a length $\eta$ boolean string containing all of $i$'s skills.

For employee $i$ to perform a job, $X_i$ needs to satisfy a boolean formula (say, given in conjunctive form). This is quite reasonable, e.g., to be an AI researcher, it needs to have knowledge of some programming language (Python, Matlab, etc.), some statistics (classical or modern), and some optimization (whether discrete or continuous). In the bipartite graph, employee $i$ has an edge to job $j$ if and only if $X_i$ satisfies this formula.

In this case, the qualifications of each employee are known by the company, who has access to their employees. Given the qualifications, the set of jobs that may be performed can be computed offline and used as advice. This advice may not be entirely correct: for example, employees may have picked up new skills (hence there may be more edges than we thought, but no less). Of course, there could also be some employees with phoney qualifications; this fraction is not too high.

One interesting property about this application is that advice may only be imperfect in the sense that edges could be added. This means that if we just mimicked, we are guaranteed to get at least $\hat{n}$. Also, the coarsening method is more easily applied.

\section{Practical considerations and extensions to our learning-augmented algorithm}
\label{sec:appendix-extensions}

\subsection{Advice coarsening}
\label{sec:appendix-advice-coarsening}

While \cref{thm:main-result} has good asymptotic guarantees as $n \to \infty$, the actual number of vertices $n$ is finite in practice.
In particular, when $n$ is ``not large enough'', \textsc{TestAndMatch} will never utilize the advice and always default to \textsc{Baseline} for all problem instances where $n \ll s_{\hat{r},\eps,\delta}$.

In practice, while the given advice types may be diverse, there could be many ``overlapping subtypes'' and a natural idea is to ``coarsen'' the advice by grouping similar types together in an effort to reduce the resultant support size of the advice (and hence $s_{\hat{r},\eps,\delta}$).
\cref{fig:coarsening-example} illustrates an extreme example where we could decrease the support size from $n$ to $2$ while still maintaining a perfect matching.

While one could treat this coarsening subproblem as an optimization pre-processing task. For completeness, we show later how one may potentially model the coarsening optimization as an integer linear program (ILP) but remark that it does not scale well in practice. That said, there are many natural scenarios where a coarsening is readily available to us.
For instance, in the online advertising, market studies typically classify users into ``types'' (with the number of types significantly less than $n$) where each type of user typically have a ``core set'' of suitable ads though the actual realized type of each arrival may be perturbed due to individual differences.

Another way to reduce the required samples for testing is to ``bucket'' the counts which are below a certain threshold to reduce the number of distinct types within the advice.
The newly created bucket type will then be a union of the types that are being grouped together.

\subsubsection{ILP for advice coarsening}

Here, we give an integer linear program (ILP) that takes in any number $|\hat{T}| = \hat{r}$ of desired groupings as input and produces a grouping proposed advice count $\hat{c}_{\hat{r}}$ on $\hat{r}$ labels that implies the maximum possible matching.
Recall that the a smaller number of resulting groups $\hat{r}$ directly translates to fewer samples $s_{\hat{r},\eps,\delta}$ required in \textsc{TestAndMatch}.
So, to utilize this ILP, one can solve for decreasing values of $r = |\hat{L}|, |\hat{L}|-1, \ldots, 1$ and evaluate the resulting maximum matching size $\hat{n}_r$ for each proposed advice count $\hat{c}_r$.
Then, one can either use the smallest possible $r$ which still preserves the size of maximum matching or even combine this with the idea from \cref{sec:patch-advice-till-perfect-matching} if one needs to further decrease $r$.

We propose to update the labels by taking \emph{intersections} of the patterns, i.e.\ for any resulting group $g_i$, we define its label pattern as $\cap_{v \in g_i} N(v)$.
Since taking intersections only restricts the edges which can be used in forming a maximum matching, this ensures that \textsc{Mimic} will always be able to mimic any proposed matching implied by the grouped patterns.

\textbf{Explanation of constants and variables}
\begin{itemize}
    \item Given the $n$ online input patterns, $b_{i,j}$ is a Boolean constant indicating whether online vertex $i \in [n]$ does \emph{not} have $j \in [n]$ as a neighbor in its pattern.
    \item Main decision variable: $x_{i,j}$ whether edge from online vertex $i$ to offline vertex $j$ is part of the matching.
    \item Auxiliary variable: $z_{i,\ell}$ is an indicator whether online vertex $i \in [n]$ is assigned to group $\ell \in [k]$.
    \item Product variable: $w_{i,j,\ell} = z_{i,\ell} \cdot z_{j, \ell}$ is an indicator whether \emph{both} online vertices $i$ and $j$ are in group $\ell$
\end{itemize}

\textbf{The ILP}
\begin{align*}
\max && \sum_{(i,j) \in E} x_{i,j}\\
s.t. && \sum_{\substack{j \in [n]\\ (i,j) \in E}} x_{i,j} & \leq 1 && \forall i \in [n] && (C1)\\
&& \sum_{\substack{i \in [n]\\ (i,j) \in E}} x_{i,j} & \leq 1 && \forall j \in [n] && (C2)\\
&& x_{i,j} & \leq 1 - w_{i,q,\ell} \cdot b_{q,j} && \forall (i,j) \in E, q \in [n], \ell \in [k] && (C3)\\
&& w_{i,j,\ell} & \leq z_{i,\ell} && \forall i \in [n], \ell \in [k] && (C4)\\
&& w_{i,j,\ell} & \leq z_{j,\ell} && \forall j \in [n], \ell \in [k] && (C5)\\
&& w_{i,j,\ell} & \geq z_{i,\ell} + z_{j,\ell} - 1 && \forall i,j \in [n], \ell \in [k] && (C6)\\
&& \sum_{\ell=1}^k z_{i,\ell} & = 1 && \forall i \in [n] && (C7)\\
&& x_{i,j} & \in \{0,1\} && \forall (i,j) \in E\\
&& z_{i,\ell} & \in \{0,1\} && \forall i \in [n], \ell \in [k]\\
&& w_{i,j,\ell} & \in \{0,1\} && \forall i,j \in [n], \ell \in [k]
\end{align*}

\textbf{Explanation of constraints}
\begin{itemize}
    \item (C1, C2) Standard matching constraints.
    \item (C3) Can only use edge $(i,j)$ if it is not ``disabled'' due to intersections. As long as \emph{some} other vertex in the same group as $i$ does \emph{not} have $j$, the edge $(i,j)$ will be disabled.
    \item (C4, C5, C6) Encoding $w_{i,j,\ell} = z_{i,\ell} \cdot z_{j, \ell}$.
    \item (C7) Every vertex assigned exactly one group.
\end{itemize}

\subsection{Remapping online arrival types}
\label{sec:appendix-remapping-online-arrival-types}

Recall the example described in \cref{sec:sigma-remapping} and how remapping helps obtain a perfect matching.
In fact, remappings can only increase the number of matches as these vertices would have been left unmatched otherwise under \textsc{Mimic}.
Note that the proposed remappings always maps an online type to a \emph{subset} so that any subsequent proposed matching can be credibly performed.

In an offline setting, given $c^*$ and $\hat{c}$, one can efficiently compute a mapping $\sigma$ that maximizes overlap using a max-flow formulation (see \cref{sec:appendix-max-flow}) and then redefine the quality of $\hat{c}$ in terms of $L_1(\sigma(c^*), \hat{c})$.
As this is impossible in an online setting, we propose a following simple mapping heuristic: when type $L$ arrives, map it to the largest subset of $A \subseteq L$ with the highest remaining possible match count.
Note that it may be the case that all subset types of $L$ no longer have a matching available to mimic from $\hat{M}$.
In the example above, we first mapped $\{u_1, u_2, u_4\}$ to $\{u_2, u_4\}$ and then to $\{u_4\}$ as $\hat{c}$ only had one count for $\{u_2, u_4\}$.

\subsubsection{Computing the optimal remapping \texorpdfstring{$\sigma$}{sigma} via a maximum flow formulation}
\label{sec:appendix-max-flow}

Consider the offline setting where we are given the true counts $c^*$ and the advice counts $\hat{c}$.

Suppose $c^*$ has $r$ non-zero counts, represented by:
$\langle L^*_1, c^*_1 \rangle, \langle L^*_2, c^*_2 \rangle, \ldots \langle L^*_r, c^*_r \rangle$, where $\sum_{i=1}^r c^*_i = n$.

Suppose $\hat{c}$ has $s$ non-zero counts, represented by:
$\langle \hat{L}_1, \hat{c}_1 \rangle, \langle \hat{L}_2, \hat{c}_2 \rangle, \ldots \langle \hat{L}_s, \hat{c}_s \rangle$, where $\sum_{i=1}^s \hat{c}_i = n$.

To compute a remapping from $c^*$ to $\hat{c}$ to maximize the number of resulting overlaps, consider the following max flow formulation on a directed graph $G = (V,E)$ with $|V| = r + s + 2$ nodes:
\begin{itemize}
    \item Create a node for each of $L^*_1, \ldots, L^*_r, \hat{L}_1, \ldots, \hat{L}_s$.
    \item Create a ``source''  and a ``destination'' node.
    \item Add an edge with a capacity $c^*_i$ from the ``source'' node to each of the $L^*_i$ nodes, for $i \in \{1, \ldots, r\}$
    \item $(\ast)$ Add an edge from $L^*_i$ to $\hat{L}_j$ with capacity $c^*_i$ if $\hat{L}_j \subseteq L^*_i$, for $i \in \{1, \ldots, r\}$ and $j \in \{1, \ldots, s\}$.
    \item Add an edge with a capacity $\hat{c}_j$ from each of the $\hat{L}_j$ nodes to the ``destination'' node, for $j \in \{1, \ldots, s\}$.
    \item Compute the maximum flow from ``source'' to ``destination''.
\end{itemize}

Since the graph has integral edge weights, the maximum flow is integral and the flow across each edge is integral.
The resultant maximum flow is the maximum attainable overlap between a remapped $c^*$ and $\hat{c}$, and we can obtain the remapping $\sigma$ by reading off the flows between on the edges from $(\ast)$.

\subsection{Augmenting advice to perfect matching}
\label{sec:appendix-augmenting-to-perfect-matching}

The following lemma tells us that there is an explicit way of augmenting $\hat{c}$ to form a new advice $\hat{c}'$ such that using $\hat{c}'$ in \textsc{TestAndMatch} does not hurt the provable theoretical guarantees as compared to directly using $\hat{c}$.

\begin{restatable}{lemma}{augmentingdoesnothurt}
\label{lem:augmenting-does-not-hurt}
Let $\hat{c}$ be an arbitrary type count with labels $\hat{T}$ implying a graph $\hat{\cG}$ with maximum matching size $\hat{n}$.
There is an explicit way to augment $\hat{c}$ to obtain $\hat{c}'$ with labels $\hat{T}'$ such that the implied graph $\hat{\cG}'$ has maximum matching size $\hat{n}' = n$.
Furthermore, running \textsc{TestAndMatch} with a slight modification of \textsc{Mimic} on $(\hat{c}', \hat{T}')$ produces a matching of size $m$ where
\[
\frac{m}{n^*}
\geq \frac{m}{n}
\geq \begin{cases}
\frac{\hat{n}}{n} - \frac{L_1(p^*, q)}{2} & \text{when $\hat{L}_1 \leq 2 \left( 1 - \beta \right) - \eps$}\\
\beta - o(1) & \text{otherwise}
\end{cases}
\]
\end{restatable}

\begin{proof}
Suppose we are given an arbitrary pattern count $\hat{c}$ and corresponding labels $\hat{L}$ such that the corresponding graph $\hat{\cG}$ has maximum matching $\hat{M}$ of size $\hat{n} < n$.
Let us fix any arbitrary maximum matching $\hat{M}$.
Denote $A_U \subseteq U$ as the set of $k = n - \hat{n}$ offline vertices and $A_V \subseteq V$ as the set of $k$ online vertices that are unmatched in $\hat{M}$.
We construct a new graph $\hat{\cG}'$ by adding a complete bipartite graph of size $k$ on $A_U \cup A_V$ to $\hat{\cG}$.
By construction, the resulting graph $\hat{\cG}'$ has a maximum matching of size $\hat{n}' = n$ due to the modified adjacency patterns of the online vertices $A_V$.

We now explain how to modify the pattern counts and labels accordingly.
Define the new set of labels $\hat{L}'$ as $\hat{L}$ with a new pattern called ``New''.
Then, we subtract away the counts of $A_V$ from $\hat{c}$ and add a count of $k$ to the label ``New'' to obtain a new pattern count $\hat{c}'$.
By construction, we see that $|\hat{L}'| = |\hat{L}| + 1$ and
\[
L_1(\hat{c}, \hat{c}')
= |\hat{c}(\text{``New''}) - \hat{c}'(\text{``New''})| + \sum_{\ell \in \hat{L}} |\hat{c}(\ell) - \hat{c}'(\ell)|
= k+k
= 2k
\]
Note that $c^*(\text{``New''}) = 0$.
By triangle inequality, we also see that
\[
L_1(c^*, \hat{c}')
\leq L_1(c^*, \hat{c}) + L_1(\hat{c}, \hat{c}')
\leq L_1(c^*, \hat{c}) + 2k
\]

\textbf{Slight modification of \textsc{Mimic}.}\;
\textsc{Mimic} will now be informed of the sets $A_U$ and $A_V$ along with the proposed matching $\hat{M}$ for the online vertices $V \setminus A_V$.
Then, whenever an online vertex $v$ arrives whose pattern does not match any in $\hat{L}$, we first try to match $v$ to an unmatched neighbor in $A_U$ if possible before leaving it unmatched.
Observe that this modified procedure can only increase the number of resultant matches since we do not disrupt any possible matchings under $(\hat{c}, \hat{L})$ while only possibly increasing the matching size via the complete bipartite graph between $A_U$ and $A_V$.

To complete the analysis, we again consider whether \textsc{Mimic} was executed throughout the online arrivals or we switched to \textsc{Baseline}, as in the analysis of \cref{thm:main-result}.
Note that now $\hat{L}_1$ is an estimate of $L_1(c^*, \hat{c}')$ instead of $L_1(c^*, \hat{c})$ and the threshold is $2 \left( \frac{\hat{n}'}{n} - \beta \right) - \eps = 2 \left( 1 - \beta \right) - \eps$ instead of $2 \left( \frac{\hat{n}}{n} - \beta \right) - \eps$ since $\hat{n}' = n$.
Also, recall that $k = n - \hat{n}$.

\textbf{Case 1}: $\hat{L}_1 < 2 (1 - \beta) - \eps$\\
Then, \textsc{TestAndMatch} executed \textsc{Mimic} throughout for all online arrivals, yielding a matching of size $m \geq n - \frac{L_1(c^*, \hat{c}')}{2}$.
Therefore,
\[
\frac{m}{n^*}
\geq \frac{m}{n}
\geq 1 - \frac{L_1(c^*, \hat{c}')}{2n}
\geq 1 - \frac{L_1(c^*, \hat{c}) + 2k}{2n}
= 1 - \frac{L_1(p,q)}{2} - \frac{n - \hat{n}}{n}
= \frac{\hat{n}}{n} - \frac{L_1(p,q)}{2}
\]

\textbf{Case 2}: $\hat{L}_1 \geq 2 (1 - \beta) - \eps$\\
Repeat the exact same analysis as in \cref{thm:main-result} but with $\hat{r}$ replaced by $\hat{r}' = |\hat{T}'| = |\hat{T}| + 1 = \hat{r} + 1$ yields a matching size of at least $\beta \cdot n - s_{\hat{r}+1,\eps,\delta} \cdot \sqrt{\log (\hat{r}+1)}$, where
\[
s_{\hat{r},\eps,\delta} = O\left( \frac{(\hat{r}+1) \cdot \log 1/\delta}{\eps^2 \cdot \log (\hat{r}+1)} \right)
\]
and $s_{\hat{r}+1,\eps,\delta} \cdot \sqrt{\log (\hat{r}+1)} \in o(1)$.
\end{proof}

\section{Deferred proofs}
\label{sec:appendix-proofs}

\begin{algorithm}[h]
\caption{$\textsc{SimulateP}$}
\label{alg:simulatep-subroutine}
\begin{algorithmic}[1]
    \INPUT Array $A$ of online arrivals so far and number of desired i.i.d.\ samples $s$, where $s \leq |A|$
    \OUTPUT $T^s_{p^*}$ \COMMENT{$s$ i.i.d.\ samples from $p^*$}
    \STATE $T^s_{p^*} \gets \emptyset$ \COMMENT{Collect simulated i.i.d.\ arrivals from $p^*$}
    \STATE $i \gets 0$
    \WHILE{$|T^s_{p^*}| < s$}
        \IF{Bernoulli($i/n$) == 1}
            \STATE $x \gets$ Pick uniformly at random from the set $\{A[0], \ldots, A[i-1]\}$
        \ELSE
            \STATE $x \gets A[i]$ \COMMENT{Uniform at random from $c^*$ under the random arrival model}
            \STATE $i \gets i + 1$
        \ENDIF
        \STATE Add $x$ to $T^s_{p^*}$
    \ENDWHILE
    \STATE \textbf{return} $T^s_{p^*}$
\end{algorithmic}
\end{algorithm}

\begin{lemma}
\label{lem:sampling-with-replacement}
In the output of \textsc{SimulateP} (\cref{alg:simulatep-subroutine}), $T^s_{p^*}$ contains $s$ i.i.d.\ samples from the realized type count distribution $p^* = c^*/n$ while using at most $s$ fresh online arrivals.
\end{lemma}
\begin{proof}
With probability $i/n$, we choose a uniform at random item from $\{A[0], \ldots, A[i-1]\}$.
With probability $1 - i/n$, we pick the next item $A[i]$ from the existing arrivals which was uniform at random under the random arrival model assumption.
Since we could possibly reuse arrivals, $T^s_{p^*}$ is formed by using at most $s$ fresh arrivals.
\end{proof}

\subsection{Proof of Theorem \ref{thm:minimax}}

\minimax*
\begin{proof}
    Using Theorem $2$ in \cite{JHW18}, we get that using $s = \Theta(\frac{r}{\varepsilon^2 \log r})$, their estimator has $\varepsilon$ additive error in expectation.
    Therefore, by using $100\cdot s$ samples, we can achive $\varepsilon/10$ additive error in expectation, i.e\ $\mathbb{E}[|\hat{L}_1 - L_1(p,q)| ]=\varepsilon/10$.
    By Markov's inequality, we get:
    \[
    \Pr[ |\hat{L}_1 - L_1(p,q)| > \varepsilon]\leq 1/10
    \]
    Thus, by repeating the entire algorithm $O(\log(1/\delta))$ times and choosing the median $\tilde{L}_1$ of the resulting estimates, we get:
     \[
    \Pr[|\tilde{L}_1 - L_1(p,q)| > \varepsilon]\leq \delta
    \]
\end{proof}

\section{Proof of concept}
\label{sec:appendix-poc}

It is our understanding that the tester proposed by \citet{JHW18} requires a significant amount of hyperparameter tuning and no off-the-shelf implementation is available \citep{JHW-email}. One may consider using an older method by \citet{valiant2011power} which is also sublinear in the number of samples but their proposed algorithm is for non-tolerant testing and requires a non-trivial code adaptation before it is applicable to $L_1$ estimation.

As a proof-of-concept, we implemented \textsc{TestAndMatch} with the empirical $L_1$ estimator and study the resultant competitive ratio under degrading advice quality.
The source code is available at \url{https://github.com/cxjdavin/online-bipartite-matching-with-imperfect-advice}.

\subsection{Implementation details}

From \cref{sec:prelims}, we know that the state-of-the-art advice-less algorithm for random order arrival is the \textsc{Ranking} algorithm of \cite{karp1990optimal} which achieve a competitive ratio of $\beta = 0.696$ \citep{mahdian2011online}.

For our testing threshold, we set $\eps = \hat{n}/n - \beta$ so that $\tau = 2(\hat{n}/n - \beta) - \eps = \hat{n}/n - \beta$.
We also implemented the following practical extensions to \textsc{TestAndMatch} which we discussed in \cref{sec:extensions}:
\begin{enumerate}
    \item Sigma remapping (\cref{sec:sigma-remapping})
    \item Bucketing so that $\hat{r}/\eps^2 < n$ (\cref{sec:coarsening-advice})
    \item Patching so that $\hat{n}' = n$ (\cref{sec:patch-advice-till-perfect-matching})
\end{enumerate}

We tested 4 variants of \textsc{TestAndMatch}, one with all extensions enabled and three others that disables one extension at a time (for ablation testing).

\subsection{Instances}

Our problem instances are generated from the synthetic hard known IID instance of \cite{manshadi2012online} where any online algorithm achieves a competitive ratio of at most 0.823 in expectation:
\begin{itemize}
    \item Let $Y_k$ denote the set of online vertices with $k$ random offline neighbors (out of $\binom{n}{k}$)
    \item Let $m = \frac{c_{2.5}^*}{2} \cdot n$, where $c_{2.5}^* = 0.81034$ is some constant defined in \cite{manshadi2012online} (not to be confused with our type counts $c^*$)
    \item Sample $m$ random online vertices from $Y_2$, i.e.\ each online vertex is adjacent to a random subset of 2 offline vertex.
    \item Sample $m$ random online vertices from $Y_3$, i.e.\ each online vertex is adjacent to a random subset of 3 offline vertex.
    \item Sample $n - 2m$ random online vertices from $Y_n$, i.e.\ each online vertex is adjacent to every offline vertex.
    \item Permute the online vertices for a random order arrival
\end{itemize}
Here, the support size of any generated type count $c^*$ is roughly $0.8n$ due to the samples from $Y_2$ and $Y_3$.

\subsection{Corrupting advice}

Starting with perfect advice $\hat{c} = c^*$, we corrupt the advice by an $\alpha$ parameter using two types of corruption.
\begin{enumerate}
    \item Pick a random $\alpha \in [0,1]$ fraction of online vertices
    \item Generate a random type for each of them by independently connecting to each offline vertex with probability $\frac{\ln n}{10 n}$.
    \item Type 1 corruption (add extra connections): Define the new type as the union of the old vertex type and the new random type.
    \item Type 2 corruption (replace connections): Define the new type as the new random type.
\end{enumerate}
As a remark, our random type generation biases towards a relatively sparse corrupted graph.

\subsection{Preliminary results}

We generated 10 random graph instances with $n = 2000$ offline and $n$ online vertices.
\cref{fig:poc} illustrates the resulting plots with error bars.

\begin{figure}[h]
    \centering
    \includegraphics[width=0.45\textwidth]{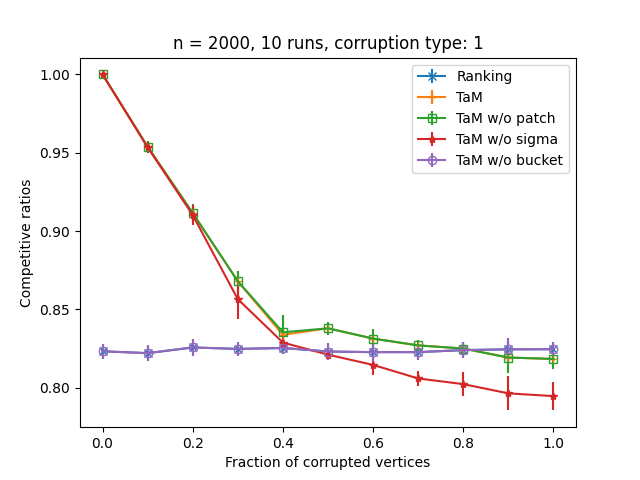}
    \includegraphics[width=0.45\textwidth]{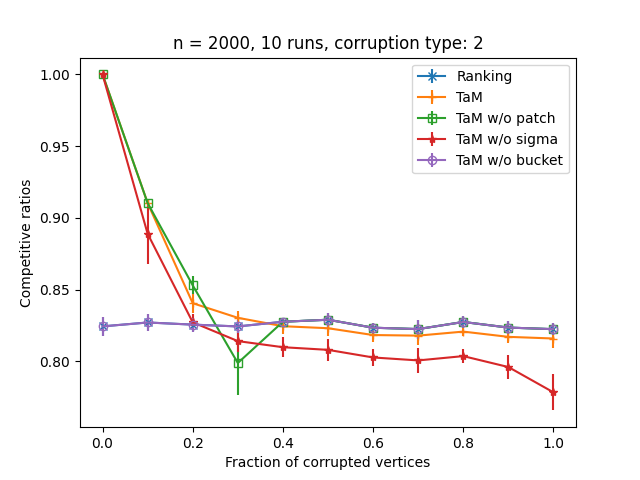}
    \caption{$n = 2000$, averaged over 10 runs. TaM refers to our implementation of \textsc{TestAndMatch}.}
    \label{fig:poc}
\end{figure}

In all cases, we see that the attained competitive ratio is highest when all extensions are enabled.
We also see that the degradation below the baseline is not very severe ($< 0.1$ for all cases, even when not all extensions are enabled).

Unsurprisingly, the competitive ratios of Ranking and ``TaM without bucket'' coincide because because $r/\eps^2 > n$ and we always default to baseline without performing any tests (to maintain robustness).

For corruption type 1, the ``sigma remapping'' extension makes our algorithm robust against additive edge corruption, and so the ``patching'' extension has no further impact.

\end{document}